\documentclass{article}

     \usepackage[preprint]{neurips_2025}

\usepackage[utf8]{inputenc} 
\usepackage[T1]{fontenc}    
\usepackage{hyperref}       
\usepackage{url}            
\usepackage{booktabs}       
\usepackage{amsfonts}       
\usepackage{nicefrac}       
\usepackage{microtype}      
\usepackage[dvipsnames]{xcolor}         

\usepackage{caption}
\usepackage{subcaption}

\usepackage{algorithm}
\usepackage{algorithmic}

\usepackage{amsmath}
\usepackage{amssymb}
\usepackage{mathtools}
\usepackage{amsthm}

\usepackage[capitalize,noabbrev]{cleveref}

\usepackage{dsfont}

\usepackage{tikz}
\usetikzlibrary{shapes.geometric}
\usetikzlibrary{positioning}
\usepackage{tikz,forest}

\usepackage{pgfplots}  
\usepgfplotslibrary{fillbetween}
\usepackage{pgfplotstable}
\pgfplotsset{compat=1.18}

\pgfplotsset{every axis/.append style={
    font=\normalsize, 
    legend cell align={left},
    grid=both,
    width=10cm,
    height=7cm,
}}

\tikzset{
  eln/.style={font = \scriptsize,circle,inner sep=3pt}
}

\forestset{
    root/.style={
        draw, fill=gray!30, rounded corners, align=center, 
        inner sep=4pt, 
    },
    treatnode/.style={
        draw, fill=green!30, ellipse, align=center, 
        inner sep=2pt, 
    },
    notreatnode/.style={
        draw, fill=red!30, ellipse, align=center, 
        inner sep=2pt, 
    },
    default tree/.style={
        for tree={
        level distance=5cm, 
        sibling distance=2cm 
        }
    },
    default preamble={
        for tree={
          scale=0.8,
        }
    },
    EL/.style = {before typesetting nodes={where n=1
        {edge label/.wrap value={node[pos=0.5,anchor=east, eln]{#1}}}
        {edge label/.wrap value={node[pos=0.5,anchor=west, eln]{#1}}}
        }}
}

\theoremstyle{plain}
\newtheorem{theorem}{Theorem}[section]

\newtheorem{lemma}[theorem]{Lemma}

\theoremstyle{definition}

\theoremstyle{remark}
\newtheorem{remark}[theorem]{Remark}

\title{Learning Treatment Allocations with Risk Control Under Partial Identifiability}

\author{%
  Sofia Ek \\
  Uppsala University\\
  \texttt{sofia.ek@it.uu.se} \\
  \And
  Dave Zachariah \\
  Uppsala University\\
  \texttt{dave.zachariah@it.uu.se} \\
}

\usepackage{amsmath,amsfonts,bm}

\newcommand{\Dec}{A}

\newcommand{\X}{X} 

\newcommand{\Loss}{L}

\newcommand{\Sampling}{S}

\newcommand{\Unobserved}{U}
\newcommand{\smallprob}{\alpha}

\newcommand{\Weight}{W}

\newcommand{\policy}{\pi}
\newcommand{\policyclass}{\Pi}
\newcommand{\odds}{\text{odds}}
\newcommand{\Weightupperbound}{\overline{\Weight}^\Gamma}

\newcommand{\Prob}{\mathbb{P}}
\newcommand{\prob}{p}
\newcommand{\probhat}{\widehat{p}}
\newcommand{\probpolicy}{p_{\policy}}
\newcommand{\Probpolicy}{\mathbb{P}_{\policy}}
\newcommand{\oddshat}{\widehat{\odds}}

\newcommand{\GammaSel}{\Gamma_{s}}
\newcommand{\GammaPol}{\Gamma_{a}}

\newcommand{\setdata}{\mathcal{D}}
\newcommand{\settrial}{\mathcal{D}}

\newcommand{\setpolicy}{\mathcal{D}_m}
\newcommand{\npolicy}{m}
\newcommand{\setB}{\mathcal{D}_l}
\newcommand{\nB}{l}
\newcommand{\setrisk}{\mathcal{D}_n}
\newcommand{\nrisk}{n}

\newcommand{\ind}[1]{\mathds{1}(#1)}

\newcommand{\rct}{\textsc{rct}}

\newcommand{\Popr}{R}
\newcommand{\Tr}{T}
\newcommand{\Trworstcase}{\overline{T}}
\newcommand{\trt}{\tau}
\newcommand{\trtparam}{t}
\newcommand{\Ehat}{\widehat{\E}}
\newcommand{\miscoveragerate}{\alpha}
\newcommand{\Trconf}{\overline{\Tr}^\miscoveragerate_{\nrisk}}

\def\eqref#1{equation~\ref{#1}}

\def\1{\bm{1}}

\DeclareMathAlphabet{\mathsfit}{\encodingdefault}{\sfdefault}{m}{sl}
\SetMathAlphabet{\mathsfit}{bold}{\encodingdefault}{\sfdefault}{bx}{n}

\newcommand{\E}{\mathbb{E}}

\DeclareMathOperator*{\argmin}{arg\,min}

\begin{document}

\maketitle

\begin{abstract}
Learning beneficial treatment allocations for a patient population is an important problem in precision medicine. Many treatments come with adverse side effects that are not commensurable with their potential benefits. Patients who do not receive benefits after such treatments are thereby subjected to unnecessary harm. This is a `treatment risk' that we aim to control when learning beneficial allocations. The constrained learning problem is challenged by the fact that the treatment risk is not in general identifiable using either randomized trial or observational data. We propose a certifiable learning method that controls the treatment risk with finite samples in the partially identified setting. The method is illustrated using both simulated and real data.
\end{abstract}

\section{Introduction}

The allocation of treatments in a patient population is a fundamental challenge in precision medicine. Consider a policy $\policy$ that recommends one of two treatment options $\Dec = \{0,1 \}$ based on a patient's observable characteristics $\X$. The recommendation may be to provide standard care $(\Dec=0)$ or administer aggressive treatment $(\Dec=1)$ among cancer patients. Here we consider dichotomous health outcomes, such as recovery or nonrecovery within a given period. Let the binary loss $\Loss \in \{0,1 \}$ indicate a non-beneficial outcome and let the proportion of health losses under policy $\policy$, that is, $\Probpolicy(\Loss=1)$, define its \emph{population risk}.

The population risk is minimized by any policy $\policy(\X)$ that assigns $\Dec=1$ to patient covariates $\X$ for which the probability of health loss is lower than under $\Dec=0$,  even if it is marginally so. For a simple illustration, consider covariate $\X \in [30,80]$ to be the age of a patient in years. Suppose the probability of health loss among patients of age $\X$ is
\begin{equation*}
\prob( \Loss = 1 | \X, \Dec = 0) = 0.80 \quad \text{and} \quad  \prob( \Loss = 1 | \X, \Dec = 1) = 0.01 \cdot (\X - 30),
\end{equation*}
for non-treated and treated, respectively. In this setting, the latter probability is lower for every $\X$ so that the treat-all policy $\policy(\X)\equiv 1$ minimizes the population risk. 

However, decision $\Dec=1$ can expose patients to adverse effects or significant side effects, such as complications, infections, or severe pain, which are not commensurable with the potential benefit $(\Loss=0)$ \citep{vickers2006decision}. Patients who receive $\Dec=1$, but no benefit $(\Loss=1)$ are thereby subjected to unnecessary harm, which may violate the principle of non-maleficence (i.e., ``above all, do no harm'') \citep{smith2005origin}. In the above example, this is the case for half of the treated 80-year-olds. We denote this as \emph{treatment risk} under policy $\policy$, that is, $\Probpolicy(\Loss=1|\Dec=1)$ or the proportion among the treated whose outcome is not beneficial. In this paper, we aim to learn policies that control the treatment risk to be no greater than some tolerated level $\trt$ (non-maleficence) while minimizing the population risk (beneficence). Varying $\trt$ determines a risk trade-off as illustrated in Figure~\ref{fig:motivating-example}. By decreasing the tolerance, the learned allocations $\policy(\X)$ should focus on those patient characteristics $\X$ that are most likely to yield benefits under treatment.

The challenge we consider in this paper is to learn a policy $\policy$ from finite data that ensures that the treatment risk is no greater than $\trt$ with high probability, even in circumstances where the risk is not point identifiable. That is, when the sampling process is not informative enough to determine a unique value of the risk \citep{manski2003identification,manski2007identification}. For observational data, this arises when there are unmeasured confounders, which in general cannot be assessed \citep{kallus2018balanced}.  For randomized trial data, this occurs when the trial and intended populations differ \citep{westreich2019epidemiology}.

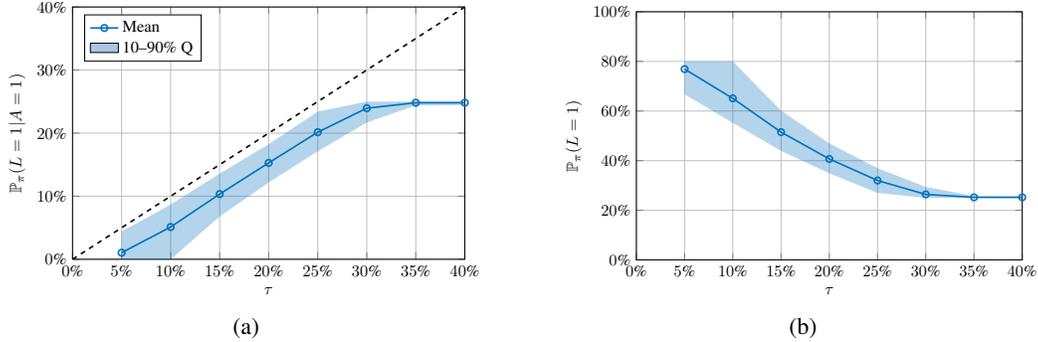
\begin{figure*}
\begin{subfigure}{0.47\textwidth}
    \centering
    \resizebox{\linewidth}{!}{ 
        \begin{tikzpicture}
    \begin{axis}[
        xlabel={$\trt$},
        ylabel={$\mathbb{P}_{\pi}(L = 1|A=1)$},
        xmin=0, xmax=0.4,
        ymin=0, ymax=0.4,
        yticklabel=\pgfmathparse{\tick*100}\pgfmathprintnumber{\pgfmathresult}\%,
        xticklabel=\pgfmathparse{\tick*100}\pgfmathprintnumber{\pgfmathresult}\%,
        legend pos=north west
    ]

    \addplot[color=RoyalBlue, mark = o, line width=1pt] table[x=beta, y=mean constr, col sep=comma] {tikz/data/syn_true1.csv}; 
    \addlegendentry{Mean}
    \addlegendimage{area legend, fill=RoyalBlue!30}
    \addlegendentry{10–90\% Q}

    \addplot[name path=upper, draw=none] table[x=beta, y=constr_q90, col sep=comma] {tikz/data/syn_true1.csv};
    \addplot[name path=lower, draw=none] table[x=beta, y=constr_q10, col sep=comma] {tikz/data/syn_true1.csv};    
    \addplot[RoyalBlue, opacity=0.3] fill between[of=upper and lower];  

    \addplot [
        domain=0.0:0.4, 
        samples=10, 
        color=black,
        line width=1pt,
        dashed
        ]
        {x};
    \end{axis}
\end{tikzpicture}
    }
    \caption{}
    \label{fig:motivating-example-criteria}
\end{subfigure} 
\hfill
\begin{subfigure}{0.47\textwidth}
    \centering
    \resizebox{\linewidth}{!}{ 
        \begin{tikzpicture}
    \begin{axis}[
        xlabel={$\trt$},
        ylabel={$\mathbb{P}_{\pi}(L = 1)$},
        yticklabel=\pgfmathparse{\tick*100}\pgfmathprintnumber{\pgfmathresult}\%,
        xticklabel=\pgfmathparse{\tick*100}\pgfmathprintnumber{\pgfmathresult}\%,
        xmin=0, xmax=0.4,
        ymin=0, ymax=1.0,
        legend pos=north east
    ]

    \addplot[color=RoyalBlue, mark = o, line width=1pt] table[x=beta, y=mean obj, col sep=comma] {tikz/data/syn_true1.csv}; 

    \addplot[name path=upper, draw=none] table[x=beta, y=obj_q90, col sep=comma] {tikz/data/syn_true1.csv};
    \addplot[name path=lower, draw=none] table[x=beta, y=obj_q10, col sep=comma] {tikz/data/syn_true1.csv};    
    \addplot[RoyalBlue, opacity=0.3] fill between[of=upper and lower];  

    \end{axis}
\end{tikzpicture}
    }
    \caption{}
    \label{fig:motivating-example-obj}
\end{subfigure}
\caption{Tolerance $\trt$ versus (a) 
 treatment risk and (b) population risk of learned allocation policies $\policy$. The proposed learning method is set to limit the treatment risk to be no greater than $\trt$ with a probability of at least $90\%$. The shaded regions (10-90th percentiles) represent the resulting risks of policies learned from 1000 different datasets. Thus $\trt$ trades off two types of risks (a) and (b). The details of the example are given in \cref{sec:exp_synthetic}.}  
\label{fig:motivating-example}
\end{figure*}

The main contribution of this work is a method for learning trustworthy treatment allocation policies that aims at reducing the population risk,
\begin{itemize}
    \item while controlling the treatment risk with a high probability in a finite sample setting, and
    \item is valid even under partial identifiability in observational or randomized trial data.
\end{itemize}

The paper is structured as follows: \cref{sec:problem} introduces the problem, followed by a discussion of its connection to existing literature in \cref{sec:background}. In \cref{sec:method}, we present our proposed policy learning method and in \cref{sec:experiments}, we evaluate its performance using both synthetic and real-world data. Finally, \cref{sec:discussion} provides a discussion of the properties of the method.

\section{Problem Formulation}
\label{sec:problem}

We want to learn a treatment allocation policy $\policy: \mathcal{X} \rightarrow \{0,1 \}$ using past data on patient covariates, treatments, and post-decision health losses $\setdata = \{ (\X_i, \Dec_i, \Loss_i) \}$. Let $\Sampling=0$ indicate samples drawn from the patient population. Then the \emph{population risk}, i.e., the proportion of health losses in the patient population, can be expressed as
\begin{equation}
\begin{split}
\Popr(\policy)&\equiv \Probpolicy(\Loss=1|\Sampling=0)\\
&= \rho_{\policy} \underbrace{ \Probpolicy(\Loss=1|\Dec=1, \Sampling=0)}_{\equiv \Tr(\policy)} + (1-\rho_{\policy}) \Probpolicy(\Loss=1|\Dec=0, \Sampling=0),
\end{split}
\label{eq:def_risks}
\end{equation}
where $\rho_{\policy}$ is the proportion treated under $\policy$ and $\Tr(\policy)$ is its \emph{treatment risk}, i.e., the  proportion of treated patients exposed to unnecessary harm. The standard risk minimization approach reduces $\Popr(\policy)$ without any control on $\Tr(\policy)$. By contrast, we may not tolerate a treatment risk greater than $\trt \in (0,1)$, in which case a treatment allocation $\policy$ that solves the constrained problem
\begin{equation}
\begin{split}
\min_{\policy \in \policyclass} \: \Popr(\policy) \quad \text{subject to} \quad \Tr(\policy) \leq \trt,
\end{split}
\label{eq:idealpolicy}
\end{equation}
is non-maleficent and makes an explicit trade-off between health outcomes \citep{wang2018learning,doubleday2022risk,kallus2022s}. For clinical decision policies, it is moreover important to restrict the learning of $\policy$ to a class of interpretable policies $\policyclass$, e.g., rule-based policies \citep{rudin2019stop}.

Let us now consider the process of obtaining training data $\setdata$ from either observational or randomized trial studies. For observational studies, the causal structure is shown in Figure~\ref{fig:DAG_obs} using directed acyclic graphs \citep{peters2017elements}, were there are  unobserved individual factors  $\Unobserved$ that may affect a patient's health loss $\Loss$ as well as the treatment assignment (aka unmeasured confounding). Then the data is drawn independently and identically (i.i.d.) from the unknown marginal distribution $\prob(\X, \Dec, \Loss|\Sampling=0) = \int \prob(\X, u, \Dec,  \Loss|\Sampling=0) du$, where the joint distribution of the process admits a causal factorization
\begin{equation}
\prob( \X, U, \Dec, \Loss|\Sampling=0) =  \underbrace{\prob( \Loss | \Dec, \X,\Unobserved)}_{\text{effect on health loss \:}} \underbrace{\prob(\Dec | \X, \Unobserved)}_{\text{treatment assignment \:}}  \underbrace{p(\X, \Unobserved | \Sampling=0)}_{\text{population characteristics}}.
\label{eq:jointdistribution_obs}
\end{equation}
For randomized trial studies, the structure is given in Figure~\ref{fig:DAG_rct}, were unobserved individual factors  $\Unobserved$ may affect the health loss $\Loss$ as well as the selection into trials (including self-selection factors). The selection is indicated by $\Sampling=1$ and leads to a trial population whose characteristics may differ from the intended patient population \citep{westreich2019epidemiology}. Then data is drawn i.i.d. from the unknown marginal distribution $\prob(\Dec, \X, \Loss|\Sampling=1) = \int \prob(\X, u, \Dec, \Loss|\Sampling=1) du$, where the joint distribution is causally factorized as
\begin{equation}
\prob( \X, \Unobserved, \Dec, \Loss|\Sampling=1) =  \prob( \Loss |  \Dec,\X,\Unobserved) \underbrace{\prob(\Dec )}_{\text{randomized assignment \:}}  \underbrace{p(\X, \Unobserved | \Sampling=1)}_{\text{trial population characteristics}}.
\label{eq:jointdistribution_rct}
\end{equation}
The risks under $\policy$ in (\ref{eq:def_risks}) are obtained from the joint distribution  
\begin{equation}
\probpolicy( \X, \Unobserved, \Dec, \Loss|\Sampling=0) =  \prob( \Loss |  \Dec, \X,\Unobserved) \; \underbrace{\ind{\Dec=\policy(\X)}}_{\text{assignment by policy \:}} \; p(\X, \Unobserved | \Sampling=0).
\label{eq:jointdistribution_policy}
\end{equation}
It is not possible to uniquely express either $\Popr(\policy)$ or $\Tr(\policy)$ in terms of the data distributions $\prob( \X,\Dec, \Loss|\Sampling=s)$ due to the causal influence of unobserved $\Unobserved$ in either (\ref{eq:jointdistribution_obs}) or (\ref{eq:jointdistribution_rct}). That is, neither the risks in (\ref{eq:def_risks}) are \emph{point identifiable} \citep{peters2017elements, westreich2019epidemiology}. Moreover, the presence of $\Unobserved$ threatens the validity of the studies, since it is in general untestable \citep{wasserman2013all}.

To tackle the policy learning problem, we consider using learned models for the treatment assignment, $\widehat{\prob}(\Dec | \X)$, and for the selection into trials, $\widehat{\prob}(\Sampling  | \X)$, when using observational or randomized trial data, respectively.  For observational studies, we employ the approach of \citet{tan2006distributional} and consider the odds of assigning treatment $\Dec$ to be miscalibrated by a factor of at most $\GammaPol$
\begin{equation}
\frac{1}{\GammaPol} \; \leq \; \underbrace{\frac{1 - \prob(\Dec | \X, \Unobserved)}{\prob(\Dec | \X, \Unobserved)}}_{\text{unknown assignment odds}} \bigg/ \underbrace{\frac{1 - \widehat{\prob}(\Dec | \X)}{\widehat{\prob}(\Dec | \X)}}_{\text{nominal assignment odds}} \; \leq \; \GammaPol.
\label{eq:Gamma_bound_a}
\end{equation}
For trial studies, we similarly consider the selection odds to be miscalibrated by a factor of at most $\GammaSel$
\begin{equation}
\frac{1}{\GammaSel} \; \leq \; \underbrace{\frac{\prob(\Sampling = 0 | \X,\Unobserved)}{\prob(\Sampling = 1 | \X,\Unobserved)}}_{\text{unknown selection odds}} \bigg/ \underbrace{\frac{\widehat{\prob}(\Sampling = 0 | \X)}{\widehat{\prob}(\Sampling = 1 | \X)}}_{\text{nominal selection odds}} \; \leq \; \GammaSel.
\label{eq:Gamma_bound_s}
\end{equation}
We shall refer to $\Gamma \geq 1$ as the \emph{degree of miscalibration}. A standard means of benchmarking a range of credible values for $\Gamma$ is to remove individual covariates $\X_j$ and treat them as unobserved $\Unobserved$ to obtain odds ratios as above \citep{manski2007identification,ichino2008temporary,huang2024sensitivity}. We demonstrate this approach in \cref{sec:app_benchmarking}. As will be shown in Lemma~\ref{lem:partialidentifiability} below, under (\ref{eq:Gamma_bound_a}) or (\ref{eq:Gamma_bound_s}) the risks $\Popr(\policy)$ and $\Tr(\policy)$ in (\ref{eq:def_risks}) can assume a range of possible values. That is, the risks are \emph{partially identifiable}. 

For any specified degree of miscalibration, the goal is to learn a policy $\policy$ that aims to minimize the population risk using $\setdata$ drawn from 
$\prob( \X, \Dec, \Loss|\Sampling=s)$ such that its treatment risk will not exceed $\tau$ with a high probability. That is, the learning method should 
produce policies that satisfy
\begin{equation}
\boxed{\Prob \big( \: \Tr(\policy) \leq \trt  \:| \: \Sampling=s \:) \geq 1-\miscoveragerate}
\label{eq:treatmentriskguarantee}
\end{equation}
at any given confidence level $1-\miscoveragerate$ and for any degree of model miscalibration up to $\Gamma$. Satisfying (\ref{eq:treatmentriskguarantee}) provides a certification of non-maleficence for trustworthy treatment allocations.

\remark While the causal structures for observational and randomized trial studies represent the most common scenarios, our framework also handles a third case in which observational studies are conducted with respect to a study population that may differ from the patient population.

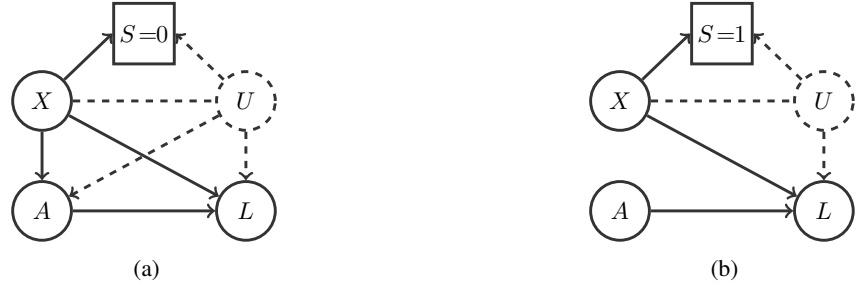
\begin{figure*}
\begin{subfigure}{0.45\linewidth}
\centering
        \resizebox{0.6\linewidth}{!}{ 
                \begin{tikzpicture}[
    square/.style={regular polygon,regular polygon sides=4},
    observednode/.style={circle, draw=black!80, fill=white!5, very thick, minimum size=8.5mm},
    unobservednode/.style={circle, draw=black!80, dashed, fill=white!5, very thick, minimum size=8.5mm},
    fixednode/.style={square, draw=black!80, fill=white!5, very thick, minimum size=5mm, inner sep = -0.12em},
    ]
    \node[fixednode]         (S)    {$\Sampling\!=\!\!0$};
    \node[observednode]      (X)    [below left = 0.2 cm and 0.7 cm of S]    {$\X$};
    \node[unobservednode]    (U)    [below right = 0.2 cm and 0.7 cm of S]   {$\Unobserved$};
    \node[observednode]      (A)    [below = 0.7 cm of X]    {$\Dec$};
    \node[observednode]      (L)    [below = 0.7 cm of U]   {$\Loss$};

    \draw[->, very thick, draw=black!80] (X) -- (A);
    \draw[->, very thick, draw=black!80] (A) -- (L);
    \draw[-, very thick, dashed, draw=black!80] (X) -- (U);
    \draw[->, very thick, dashed, draw=black!80] (U) -- (S.east);
    \draw[->, very thick, dashed, draw=black!80] (U) -- (L);
    \draw[->, very thick, draw=black!80] (X) -- (L);
    \draw[<-, very thick, draw=black!80] (S.west) -- (X);
    \draw[->, very thick, dashed, draw=black!80] (U) -- (A);
\end{tikzpicture}  
            }
        \caption{}
        \label{fig:DAG_obs} 
\end{subfigure}
\hfill
\begin{subfigure}{0.45\linewidth}
\centering
        \resizebox{0.6\linewidth}{!}{ 
                \begin{tikzpicture}[
    square/.style={regular polygon,regular polygon sides=4},
    observednode/.style={circle, draw=black!80, fill=white!5, very thick, minimum size=8.5mm},
    unobservednode/.style={circle, draw=black!80, dashed, fill=white!5, very thick, minimum size=8.5mm},
    fixednode/.style={square, draw=black!80, fill=white!5, very thick, minimum size=5mm, inner sep = -0.12em},
    ]
    \node[fixednode]         (S)    {$\Sampling\!=\!\!1$};
    \node[observednode]      (X)    [below left = 0.2 cm and 0.7 cm of S]    {$\X$};
    \node[unobservednode]    (U)    [below right = 0.2 cm and 0.7 cm of S]   {$\Unobserved$};
    \node[observednode]      (A)    [below = 0.7 cm of X]    {$\Dec$};
    \node[observednode]      (L)    [below = 0.7 cm of U]   {$\Loss$};

    \draw[->, very thick, draw=black!80] (A) -- (L);
    \draw[-, very thick, dashed, draw=black!80] (X) -- (U);
    \draw[->, very thick, dashed, draw=black!80] (U) -- (S.east);
    \draw[->, very thick, dashed, draw=black!80] (U) -- (L);
    \draw[->, very thick, draw=black!80] (X) -- (L);
    \draw[<-, very thick, draw=black!80] (S.west) -- (X);
\end{tikzpicture}  
            }
        \caption{}
        \label{fig:DAG_rct} 
\end{subfigure}
    \caption{Structural causal models, specified by acyclic directed graphs, that describe the data generating process in (a) observational studies and (b) randomized trials (b). In (a), both observed covariates $\X$ and unbserved factors $\Unobserved$ may jointly influence treatment decisions $\Dec$ and health loss $\Loss$, introducing unmeasured confounding. In (b), unobserved factors $\Unobserved$ may additionally influence  selection into trials, introducing unmeasured selection factors. The (conditional) indicator $\Sampling$ determines inclusion in a randomized trial study.}
    \label{fig:DAG_policy}      
\end{figure*}

\section{Background}
\label{sec:background}
Treatment effect estimation has traditionally focused on the average treatment effects (ATE) across populations, typically using data from randomized controlled trials \citep{imbens2015causal}. Optimal treatment allocation is instead focused on individual-level heterogeneity to assign treatments to those are most likely to benefit, see for example; \citet{ manski2004statistical, dudik2011doubly, qian2011performance, zhang2012robust, zhao2012estimating, swaminathan2015counterfactual, athey2016recursive, wager2018estimation, nie2021quasi}, or the overview in \citet{hoogland2021tutorial} for the case of binary outcomes. An important recent extension of this problem formulation is policy learning with continuous health outcomes under partial identifiability \citep{kallus2021minimax, cui2021individualized, christensen2023optimal, adjaho2023externally, yata2025optimal, ben2025safe}.

Another related line of work is policy learning under constraints \citep{kitagawa2018should,athey2021policy}. Specifically, \citet{wang2018learning} and \citet{doubleday2022risk} consider problems in which a primary health outcome and a secondary adverse health variable are observed. They learn linear or decision-rule policies that maximize the benefit while constraining the expected adverse health outcome, i.e., the specified harm. The proposed learning methods are, however, not certified to control the harm. Provided such secondary health losses, the method developed herein can readily control the harm with confidence. A different notion of harm, called the `fraction of negatively affected', was defined in \citet{kallus2022s}. It is based on the binary \emph{potential} health outcomes for each patient under hypothetical $\Dec=0$ and $\Dec=1$ \emph{simultaneously}. While this counterfactual harm is unobservable, a novel bound and asymptotically valid confidence interval for it was derived by \citet{kallus2022s}. The bound was tightened by \citet{li2023trustworthy}, who employed it in a harm constrained policy learning method that achieves asymptotic harm control. However, the bound assumes that the potential outcomes are nonnegatively correlated, which is unverifiable since such counterfactuals can never be simultaneously observed (cf. \citet{sarvet2023perspective}).

Simple treatment policies play an important role in ensuring effective, transparent, and scalable care. More complex data-driven models can often optimize treatment decisions, but tend to lack interpretability and, importantly, may be difficult to implement in practice \citep{rudin2019stop, athey2021policy}. Simple policies, on the other hand, enhance clinical interpretability, allowing healthcare providers to understand and apply treatment recommendations with confidence \citep{kitagawa2018should}. This is particularly important in high-stakes medical settings where decisions must be made reliably without delay \citep{caruana2015intelligible}.

\citet{bates2021distribution} constructs prediction sets that provide finite-sample guarantees on coverage probability, without requiring strong distributional assumptions about the data. We extend this technique to the policy setting to give guarantees on controlling the treatment risk, even in the case of confounding or selection bias.

\section{Method}
\label{sec:method}
We will now propose a method to find a policy $\policy(\X)$ that satisfies \cref{eq:treatmentriskguarantee} with a user-specified treatment risk tolerance $\trt$. For notational convenience, we let $\Ehat_{\npolicy}[Z]$ denote the empirical mean, $\frac{1}{\npolicy}\sum_{i=1}^{\npolicy}Z_i$. The proofs of the results presented here appear at the end of the section.

Since the risks in (\ref{eq:def_risks}) are not point identifiable, we first derive their upper bounds using the data distributions for any given degree of miscalibration $\Gamma \geq 1$. 
\begin{lemma}\label{lem:partialidentifiability}
The population and treatment risks are upper bounded by
\begin{equation}
\Popr(\policy) \leq \E\left[\Loss \cdot \Weightupperbound \big|\Sampling=s\right] \quad \text{and} \quad \Tr(\policy) \leq \E \left[L \cdot \frac{\ind{\Dec =1}}{\probpolicy(\Dec=1|\Sampling=0)} \cdot \overline{\Weight}^{\Gamma} \big| \Sampling=s\right],
\label{eq:upperbounds}
\end{equation}
where $\overline{\Weight}^{\Gamma}$ denote  weights. These are given by
\begin{equation}
\overline{\Weight}^{\Gamma} = \left[1 + \GammaPol  \big(\probhat(\Dec| \X)^{-1} - 1 \big) \right] 
\cdot \ind{\Dec = \policy(X)},
\label{eq:weights_upper_a}
\end{equation}
in the case of observational data $(\Sampling=0)$,
and
\begin{equation}
\overline{\Weight}^{\Gamma} =
\GammaSel \cdot
\frac{\widehat{\prob}(\Sampling = 0 | \X)}{\widehat{\prob}(\Sampling = 1 | \X)} \cdot
\frac{\prob(\Sampling = 1)}{\prob(\Sampling = 0)} \cdot
\frac{\ind{\Dec = \policy(X)}}{\prob(\Dec|\X)},
\label{eq:weights_upper_s}
\end{equation}
when using randomized data $(\Sampling=1)$. Moreover, if $\Gamma = 1$, then (\ref{eq:upperbounds}) holds with equalities.
\end{lemma}

\begin{remark}
Since the policy only takes $\X$ as an input, its treatment probability $\probpolicy(\Dec=1|\Sampling=0) = \int 
\ind{1=\policy(x)} \; p(x | \Sampling=s) dx$ is point identifiable.
\end{remark}

To achieve (\ref{eq:treatmentriskguarantee}), we apply sampling splitting by randomly splitting $\settrial$ into two parts $\setpolicy$, and $\setrisk$ with $\npolicy$ and $\nrisk$ samples, respectively. The first set $\setpolicy$ is used to form an empirical analogue of (\ref{eq:idealpolicy}), along with the upper bounds (\ref{eq:upperbounds}), to learn a nominal policy  
\begin{equation}
\begin{aligned}
\policy(\X; t) = \argmin_{\policy \in \Pi} \: \Ehat_{\npolicy} \left[L \cdot \overline{\Weight}^{\Gamma}\right]  \;
    \textrm{subject to } \; \Ehat_{\npolicy}\left[L \cdot \frac{\ind{\Dec =1}}{\probpolicy(\Dec=1|\Sampling=0)} \cdot \overline{\Weight}^{\Gamma}\right] \leq \trtparam, 
\end{aligned}
\label{eq:minimization_policy}
\end{equation}
where $\trtparam \in (0,1)$ is a nominal tolerance parameter. Thus (\ref{eq:minimization_policy}) yields a mapping from $(\trtparam, \setpolicy, \Gamma )$ to a nominal policy $\policy(\X; \trtparam)$ and we let $\Trworstcase(\trtparam) = \E[V(\trtparam)|\Sampling=s]$ denote the bound on its treatment risk in (\ref{eq:upperbounds}), where  $ V(\trtparam) = L \cdot \frac{\ind{\Dec =1}}{\probpolicy(\Dec=1|\Sampling=0)} \cdot \overline{\Weight}^{\Gamma}(\trtparam)$. 

The second dataset $\setrisk$ is now used to construct an upper confidence bound of its treatment risk, denoted $\Trconf(\trtparam)$, such that 
\begin{equation}
    \Prob(\Trworstcase(\trtparam) \leq \Trconf(\trtparam)|\Sampling=s) \geq 1 - \miscoveragerate.
    \label{eq:upperconfidencebound}
\end{equation}

\begin{theorem}
\label{sec:main_theorem}
Let $\Trconf(\trtparam)$ denote an upper bound that satisfies (\ref{eq:upperconfidencebound}). For any specified treatment risk tolerance $\trt$, define the empirical tolerance
\begin{equation}
\trtparam_{\nrisk} = \argmin_{t \in (0,1)} \; \Ehat_{\nrisk} \left[L \cdot \overline{\Weight}^{\Gamma}(t)\right] \quad \text{subject to} \; \trt > \Trconf(\trtparam'), \quad \forall \trtparam' \leq \trtparam.
    \label{eq:choose_trt} 
\end{equation}
If the nominal $\policy(\X; \trt)$ obtained from  (\ref{eq:minimization_policy}) achieves the constraint with equality, then $\policy(\X; \trtparam_n )$ is certified to control the treatment risk according to (\ref{eq:treatmentriskguarantee}). That is, it satisfies
    $\Tr(\policy) \leq \trt$
with probability no less than $1-\miscoveragerate$ under all degrees of miscalibration up to $\Gamma$. 
\end{theorem}

\begin{remark} There are several possible confidence bounds (\ref{eq:upperconfidencebound}). One method is based on Bentkus inequality which becomes tight for binary data \citep{bates2021distribution}. It is obtained as follows: Let $V_{\text{max}}$ denote an upper limit on $V(\trtparam)$ which contains the binary $\Loss$. We therefore expect $V(\trtparam)$ to be either 0 or clustered closer around some values  towards $V_{\text{max}}$. A confidence bound $\Trconf(\trtparam)$ that is tight for a binary loss is useful in this case. (With a smaller variance, other bounds may be tighter, see \citep{bates2021distribution}.) Define    
\begin{equation*}
    g(a;\Trworstcase(\trtparam)) =\min \left( \; 
    \exp\{ -nh(a;\Trworstcase(\trtparam))     \}, \;
    e \cdot \textsc{cdf}(\lceil \nrisk a \rceil  ; n,\Trworstcase(\trtparam) ) \;\right), 
\end{equation*}
where
\begin{equation*}
    h(a;\Trworstcase) = a \log(a/\Trworstcase) + (1-a)\log((1-a)/(1-\Trworstcase)),
\end{equation*}
and $\textsc{cdf}(\cdot ; n,p )$ is the cumulative distribution function of a binomial random variable with sample size $\nrisk$ and success probability $p$.
Then 
\begin{equation}
    \Trconf(\trtparam) = \sup\left\{\Trworstcase: g\left(\frac{\Ehat_{\nrisk}[V(\trtparam)]}{V_{\text{max}}}; \frac{\Trworstcase}{V_{\text{max}}}\right) \geq \miscoveragerate \right\},
    \label{eq:bound_bentkus}
\end{equation}
is a valid upper confidence bound (\ref{eq:upperconfidencebound}). The complete method is summarized in \cref{alg:policy}.
\end{remark}

\begin{remark} The empirical tolerance $\trtparam_{\nrisk}$ is chosen to yield the lowest bound on the in-sample population risk while satisfying the constraint with respect to the upper confidence bound.
\end{remark}

\begin{remark} An alternative method to construct a policy  $\policy(\X)$ that fulfills $\Tr(\policy) \leq \trt $ with a high probability in average is presented in \cref{sec:app_alt_method}.
\end{remark}

\begin{algorithm}[H]
    \caption{Learn $\policy(\X)$}     
    \label{alg:policy}
    \begin{algorithmic}[1]
        \INPUT Data $\setdata$, degree of miscalibration $\Gamma$, policy class $\Pi$, parameter $\trt$, confidence level $1-\smallprob$.
        \OUTPUT Policy $\policy(\X)$
        \STATE Randomly split $\settrial$ into $\setpolicy$, and $\setrisk$.
        \FOR {$\trtparam \in (0,1)$}
            \STATE Learn $\policy(\X; t)$ as in \cref{eq:minimization_policy} using $\setpolicy$.
            \STATE Save $\policy(\X; t)$.
        \ENDFOR
        \FORALL{saved $\policy(\X; t)$}
        \STATE Compute upper bound $\Trconf(\trtparam)$ from \cref{eq:upperconfidencebound}, e.g., via \cref{eq:bound_bentkus}, using $\setrisk$.
        \ENDFOR
        \STATE Select $\trtparam_n$ as in \cref{eq:choose_trt} using $\setrisk$.
    \end{algorithmic}
\end{algorithm}

\begin{proof}[Proof of \cref{lem:partialidentifiability}]
We start by proving the bound for $\Tr(\policy)$. Since $\Loss \in \{0,1 \}$, we have that
\begin{equation*}
\Tr(\policy) = \Prob_{\policy}(\Loss=1|A = 1, \Sampling=0) = \E_{\policy}[\Loss|A = 1,\Sampling=0] = \sum_{\ell} \ell \cdot  p_{\policy}(\ell|\Dec = 1,\Sampling=0),
\end{equation*}
and, moreover,  using the chain rule: $p_{\policy}(\ell|A = 1,\Sampling=0) = p_{\policy}(\ell,A = 1|\Sampling=0) / p_{\policy}(A = 1|\Sampling=0)$, it follows that
\begin{equation*}
\begin{split}
\Tr(\policy)&= \sum_\ell \sum_{a} \ell \frac{p_{\policy}(\ell, a|\Sampling=0)}{p_{\policy}(\Dec = 1|\Sampling=0)} \ind{a = 1} = \E_{\policy}\left[ \Loss \frac{\ind{\Dec =1}}{p_{\policy}(\Dec=1|\Sampling=0)} | \Sampling = 0\right].    
\end{split}
\end{equation*}
Next, we note that
\begin{equation*}
   \E_{\policy}[Z|\Sampling=0] = \E \left[ Z \cdot \frac{\probpolicy( \X, \Unobserved,  \Dec,\Loss|\Sampling=0) }{\prob( \X, \Unobserved, \Dec, \Loss|\Sampling=s)} |\Sampling=s \right] = \E \left[ Z \cdot \Weight_\policy |\Sampling=s \right],
\end{equation*}
where $\Weight_\policy = \probpolicy(\X, \Unobserved, \Dec,  \Loss|\Sampling=0) / \prob( \X, \Unobserved, \Dec, \Loss|\Sampling=s)$ is an importance weight  given by (\ref{eq:jointdistribution_obs}) or (\ref{eq:jointdistribution_rct}), depending on whether we are using observational or trial data. The weight is upper bounded $\Weight_\policy \leq \overline{\Weight}^{\Gamma}$  from  (\ref{eq:weights_upper_a}) or (\ref{eq:weights_upper_s}), using (\ref{eq:Gamma_bound_a}) or (\ref{eq:Gamma_bound_s}) correspondingly. This proves the upper bound for $\Tr(\policy)$. 
The bound for $\Popr(\policy)$ is analogous.
\end{proof}

\begin{proof}[Proof of \cref{sec:main_theorem}]
For notational convenience, we drop the symbol of conditioning on $\Sampling=s$ in the expressions that follow. Using $\trtparam_n$ in (\ref{eq:choose_trt}), we want to ensure that 
\begin{equation}
\Prob( \Trworstcase(\trtparam_n) > \trt ) \leq \miscoveragerate.
\label{eq:thmclaim}
\end{equation}
From (\ref{eq:upperconfidencebound}), we have that
\begin{equation}
\Prob( \Trworstcase(\trt) > \Trconf(\trt) ) \leq \miscoveragerate.
\label{eq:upperconfidence_alt}
\end{equation}
For policy $\policy(\X, \trtparam)$, which satisfies the constraint in (\ref{eq:minimization_policy}),  $\Ehat_{\npolicy}[V(\trtparam)] = \frac{1}{\npolicy}\sum^{\npolicy}_{i=1} V_i(t) \leq \trtparam$ so that after applying an expectation on both sides of the inequality we have
\begin{equation}
t \geq \frac{1}{\npolicy}\sum^{\npolicy}_{i=1} \E[V_i(t)] = \E[V(t)] = \Trworstcase(\trtparam).
\label{eq:monotonicity}
\end{equation}
By construction of the empirical tolerance $\trtparam_{\nrisk}$ in (\ref{eq:choose_trt}), we have that
\begin{equation}
\forall t \leq \trtparam_{\nrisk} \; : \; \trt > \Trconf(t).
\label{eq:parameterproperty}
\end{equation}
For the event under consideration $ \Trworstcase(\trtparam_n) > \trt$, it follows that  $\trtparam_{\nrisk} \geq \Trworstcase(\trtparam_{\nrisk}) > \trt \geq \Trworstcase(\trt)$ using (\ref{eq:monotonicity}). 
Therefore, $\trtparam_{\nrisk} > \trt$ and by (\ref{eq:parameterproperty}) we have $\trt > \Trconf(\trt)$. Since $\policy(\X; \trt)$ yields
$\Trworstcase(\trt) = \E\left[ \Ehat_{\npolicy}[V(\trt)] \right] = \trt$ via (\ref{eq:monotonicity}), it follows that
\begin{equation}
\Trworstcase(\trt) > \Trconf(\trt).
\end{equation}
By (\ref{eq:upperconfidence_alt}), this event occurs with a probability of at most $\miscoveragerate$. The proof technique is similar to that employed in \citep{bates2021distribution}.

\end{proof}

\section{Experimental}
\label{sec:experiments}
To illustrate the properties of the proposed method of treatment risk control, we use both synthetic and real-world datasets. For concreteness, we consider the policy class $\policyclass$ in (\ref{eq:minimization_policy}) to be a family of fast-and-frugal decision trees \citep{gigerenzer2000simple}. The learning problem is solved following the greedy approach proposed in \citet{zhang2015using} with the key difference that the constraint is also evaluated in each search step. Further details are provided in \cref{sec:fft}.

\subsection{Synthetic Data}
\label{sec:exp_synthetic}

We begin by describing an observational data distribution
(\ref{eq:jointdistribution_obs}). In the first case, there is no $\Unobserved$ that can affect both treatment and health outcomes. The covariates $\X$ are two-dimensional and $\prob(\X|\Sampling=0)$ is given by
\begin{equation}
    \X = \begin{bmatrix} \X_1 \\ \X_2 \end{bmatrix} |\Sampling=0 \sim \mathcal{U}(30,80)^2.
    \label{eq:syn_gen_x}
\end{equation} 
The assignment of treatment actions $\Dec \in \{0, 1\}$ follow a known distribution 
\begin{equation}
     \prob(\Dec=1 |\X) = \sigma\left(0.5 - \frac{\X_1 - 30}{50}\right) \quad \text{where }\sigma(x) = \frac{1}{1 + \exp{(-x)}}.
     \label{eq:synthetic_propensity}
\end{equation} 
The health loss probability for both treatment options follows
\begin{equation}
\begin{aligned}
    \prob(\Loss=1 |\X, \Dec = 0) = 0.8 \quad \text{and} \quad
    \prob(\Loss=1 |\X, \Dec = 1) = 0.01 \cdot (\X_1 - 30).
\end{aligned}
\label{eq:syn_gen_y}
\end{equation}
We generate  $|\setdata| = \npolicy+\nrisk=1000+1000$ samples. The set $\setdata_{\npolicy}$ is used to learn $\policy(\X; \trtparam)$ in (\ref{eq:minimization_policy}), where $\trtparam \in (0, 0.5]$  evaluated at 200 equally spaced points. The decision tree is restricted to a single split, where the covariates are discretized into 200 bins. The other $\setdata_{\nrisk}$ is used to form $\trtparam_{\nrisk}$ in (\ref{eq:choose_trt}) using a miscoverage rate of $\miscoveragerate = 10\%$ and specified tolerance $\trt$. The resulting policy is $\policy(\X; \trtparam_{\nrisk})$.

We evaluate the population and treatment risks for 1000 different policies learned from different draws of $\setdata$. Initially, we assume no miscalibration of the treatment assignment odds $(\Gamma=1)$. In \cref{fig:motivating-example-criteria}, it is evident that $\trt$ effectively control the treatment risk and that at least $1-\miscoveragerate$ of the runs are below the tolerance $\tau$. When comparing with \cref{fig:motivating-example-obj}, the trade-off between population and treatment risks is clearly visible. Additional results for $\Gamma=2$ are provided in \cref{sec:app_synthetic}. For illustration purposes, learned certified policies $\policy(\X; \trtparam_n)$ for three different risk tolerances $\trt$ are visualized in \cref{fig:fft_synthetic}. As expected, a lower tolerance $\trt$ results in treatments being assigned to individuals with lower $X_1$. As $\trt$ increases, the proportion of treated individuals also increases.

We now add an unmeasured confounding variable $\Unobserved$, uniformly distributed $\Unobserved \sim \mathcal{U}(0,1)$. The treatment assignment probability is set to
\begin{equation*}
\begin{aligned}
    \prob(\Dec=1 |\X, \Unobserved) = \ind{\Unobserved < 0.5} \cdot \frac{\sigma_{\text{nom}}(\X)}{2 - \sigma_{\text{nom}}(\X)} 
     + \ind{\Unobserved \geq 0.5} \cdot \frac{\sigma_{\text{nom}}(\X)}{0.5 + 0.5\sigma_{\text{nom}}(\X)},
\end{aligned}
\end{equation*} 
where $\sigma_{\text{nom}}(\X)$ is the nominal probability model in (\ref{eq:synthetic_propensity}). The loss for the treated is drawn from
\begin{equation}
\begin{aligned}
    \prob(\Loss=1 |\Dec = 1, \X, \Unobserved) = \ind{\Unobserved < 0.5} \cdot 0.02 \cdot (\X_1 - 30)
    + \ind{\Unobserved \geq 0.5} \cdot 0.002 \cdot (\X_1 - 30).
\end{aligned}
\label{eq:syn_loss_confounding}
\end{equation}

\begin{figure*}
\begin{subfigure}{0.33\linewidth}
    \centering
    \begin{forest}
        [$X_1 < 40.7$ years, root
            [{$A = 1$}, treatnode, EL={yes}]
            [{$A = 0$}, notreatnode, EL={no}]
        ]
    \end{forest}
    \vspace{3pt}
    \caption{$\trt = 10\%$}
\end{subfigure}
\begin{subfigure}{0.33\linewidth}
    \centering
    \begin{forest}
        [$X_1 < 55.7$ years, root
            [{$A = 1$}, treatnode, EL={yes}]
            [{$A = 0$}, notreatnode, EL={no}]
        ]
    \end{forest}
    \vspace{3pt}
    \caption{$\trt = 20\%$}
\end{subfigure}
\begin{subfigure}{0.33\linewidth}
    \centering
    \begin{forest}
        [$X_1 < 73.4$ years, root
            [{$A = 1$}, treatnode, EL={yes}]
            [{$A = 0$}, notreatnode, EL={no}]
        ]
    \end{forest}
    \vspace{3pt}
    \caption{$\trt = 30\%$}
\end{subfigure}
\caption{Treatment allocation policies $\policy(\X; \trtparam_{\nrisk})$ learned from a synthetic dataset under different risk tolerance levels $\trt$, assuming $\Gamma = 1$ in \cref{eq:minimization_policy}.}

\label{fig:fft_synthetic}
\end{figure*}
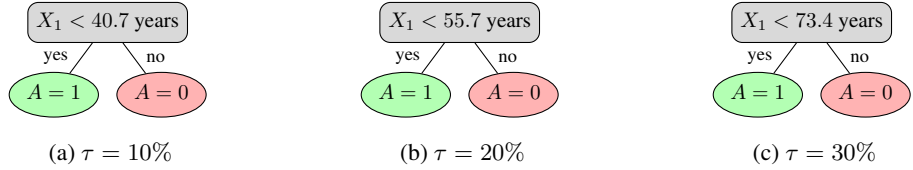

We assume that the learning method uses $\probhat(\Dec|\X) = \sigma_{\text{nom}}(\X)$ as a nominal model, which corresponds to an odds miscalibration degree of $\Gamma = 2$ in (\ref{eq:Gamma_bound_a}). We repeat the evaluation of the learned policies with $1 \leq \Gamma \leq 2$. In \cref{fig:synthetic-vio-criteria}, it is evident that the treatment risk cannot be controlled when assuming no miscalibration ($\Gamma = 1$). In contrast, if the odds ratios could be off by a factor of up to $\Gamma = 2$, the treatment risk is indeed controlled. In \cref{fig:synthetic-vio-obj}, the trade-off between minimizing the probability of nonrecovery and minimizing the treatment risk is still visible, and using $\Gamma = 2$ results in a higher probability of nonrecovery since the treatment risk constraint is tighter. Additional results in the case of randomized trial data are included in \cref{sec:app_synthetic}.

\begin{figure*}
\begin{subfigure}{0.47\textwidth}
    \centering
    \resizebox{\linewidth}{!}{ 
        \begin{tikzpicture}
    \begin{axis}[
        xlabel={$\trt$},
        ylabel={$\mathbb{P}_{\pi}(L = 1|A=1, S=0)$},
        yticklabel=\pgfmathparse{\tick*100}\pgfmathprintnumber{\pgfmathresult}\%,
        xticklabel=\pgfmathparse{\tick*100}\pgfmathprintnumber{\pgfmathresult}\%,
        xmin=0, xmax=0.4,
        ymin=0, ymax=0.4,
        legend pos=north west
    ]

    \addplot[color=RoyalBlue, mark = o, line width=1pt] table[x=beta, y=mean constr, col sep=comma] {tikz/data/syn_gamma1.csv}; 
    \addlegendentry{$\Gamma = 1$}

    \addplot[color=Peach, mark = star, line width=1pt] table[x=beta, y=mean constr, col sep=comma] {tikz/data/syn_gamma2.csv}; 
    \addlegendentry{$\Gamma = 2$}

    \addplot[name path=upper, draw=none] table[x=beta, y=constr_q90, col sep=comma] {tikz/data/syn_gamma1.csv};
    \addplot[name path=lower, draw=none] table[x=beta, y=constr_q10, col sep=comma] {tikz/data/syn_gamma1.csv};    
    \addplot[RoyalBlue, opacity=0.3] fill between[of=upper and lower];  

    \addplot[name path=upper, draw=none] table[x=beta, y=constr_q90, col sep=comma] {tikz/data/syn_gamma2.csv};
    \addplot[name path=lower, draw=none] table[x=beta, y=constr_q10, col sep=comma] {tikz/data/syn_gamma2.csv};    
    \addplot[Peach, opacity=0.3] fill between[of=upper and lower];  

    \addplot [
        domain=0.0:0.4, 
        samples=10, 
        color=black,
        line width=1pt,
        dashed
        ]
        {x};
    \end{axis}
\end{tikzpicture}
    }
    \caption{}
    \label{fig:synthetic-vio-criteria}
\end{subfigure}
\hfill  
\begin{subfigure}{0.47\textwidth}
    \centering
    \resizebox{\linewidth}{!}{ 
        \begin{tikzpicture}
    \begin{axis}[
        xlabel={$\trt$},
        ylabel={$\mathbb{P}_{\pi}(L = 1|S=0)$},
        yticklabel=\pgfmathparse{\tick*100}\pgfmathprintnumber{\pgfmathresult}\%,
        xticklabel=\pgfmathparse{\tick*100}\pgfmathprintnumber{\pgfmathresult}\%,
        xmin=0, xmax=0.4,
        ymin=0, ymax=1.0,
        legend pos=north east
    ]

    \addplot[color=RoyalBlue, mark = o, line width=1pt] table[x=beta, y=mean obj, col sep=comma] {tikz/data/syn_gamma1.csv}; 
    

    \addplot[color=Peach, mark = star, line width=1pt] table[x=beta, y=mean obj, col sep=comma] {tikz/data/syn_gamma2.csv}; 
    

    \addplot[name path=upper, draw=none] table[x=beta, y=obj_q90, col sep=comma] {tikz/data/syn_gamma1.csv};
    \addplot[name path=lower, draw=none] table[x=beta, y=obj_q10, col sep=comma] {tikz/data/syn_gamma1.csv};    
    \addplot[RoyalBlue, opacity=0.3] fill between[of=upper and lower];  

    \addplot[name path=upper, draw=none] table[x=beta, y=obj_q90, col sep=comma] {tikz/data/syn_gamma2.csv};
    \addplot[name path=lower, draw=none] table[x=beta, y=obj_q10, col sep=comma] {tikz/data/syn_gamma2.csv};    
    \addplot[Peach, opacity=0.3] fill between[of=upper and lower];  

    \end{axis}
\end{tikzpicture}
    }
    \caption{}
    \label{fig:synthetic-vio-obj}
\end{subfigure}
\caption{The treatment risk tolerance $\trt$ versus treatment risk, $\Tr(\policy)$, and the population risk, $\Popr(\policy)$, under $\policy$. The learned policies are certified to control the risk with probability of at least $1-\miscoveragerate = 90\%$, up to a specified degree of miscalibration ($\Gamma = {1 ,2}$). The shaded blue region represents the range between the 10th and 90th percentiles of policies learned using 1,000 different datasets. (a) For an assumed degree of odds miscalibration of $\Gamma = 2$, the treatment risk is controlled as expected. Assuming no miscalibration, $\Gamma=1$, is less credible and we also see there is no risk control in this case. (b) The corresponding population risks.}
\label{fig:synthetic-vio}
\end{figure*}
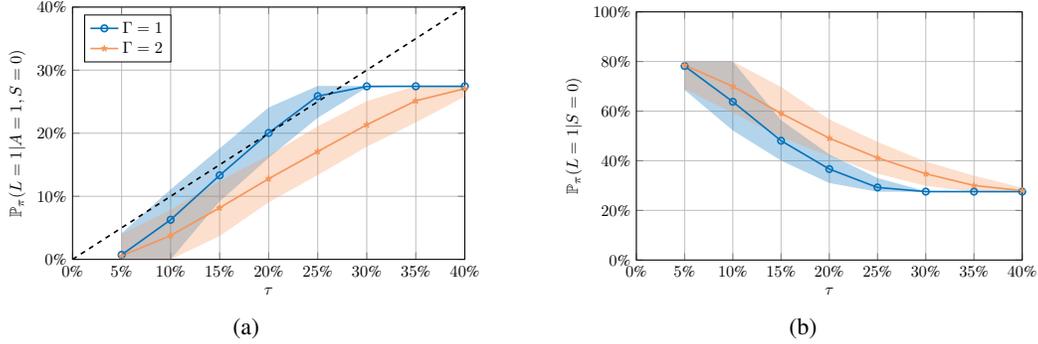

\subsection{STAR Data}
\label{sec:star}
For illustration, we also test our method on real-world data from the Tennessee Student/Teacher Achievement Ratio (STAR) study \citep{DVN/SIWH9F_2008, krueger1999experimental}, a randomized controlled trial on class size conducted between 1985 and 1989. In this study, preschool through third-grade students and their teachers were randomly assigned to one of three class types: small (13–17 students per teacher), regular (22–25 students per teacher) and regular with an aide (22–25 students per teacher plus a full-time assistant). However, in our analysis, we focus only on the first two groups. Each student is characterized by $\X$, which covers 11 covariates such as birth month, gender, teacher career, and experience. Additional details can be found in \cref{sec:app_star}.

Following \citet{kallus2018removing}, we define first-grade class type as the treatment, as many students were not enrolled in the study during preschool ($\Dec =0$ is `regular class size' and $\Dec =1$ is `small class'). The outcome variable $\Loss$ is the achievement test score at the end of first grade, calculated as the sum of the standardized math, reading and listening scores. This score is then binarized, with $\Loss = 0$ indicating a sum above the median and $\Loss = 1$ indicating a sum at or below the median. Students who were not part of the STAR study in first grade, had missing outcome data or were assigned to the regular class with an aide were excluded.

The final dataset consists of 4218 students. A 50 percent of the samples is randomly split for policy construction in (\ref{eq:minimization_policy}), where $\trtparam$ is evaluated at 100 equally spaced points in the range $(0, 0.8]$. The decision tree is constrained to a maximum of three splits, and similarly to the synthetic case, continuous covariates are discretized into 200 bins. The next 25 percent of samples are used to form (\ref{eq:choose_trt}) for the certified policy $\policy(\X; \trtparam_{\nrisk})$. The remaining 25 percent of samples is used for evaluation. For the STAR data, a full evaluation, as conducted in the synthetic case, is not feasible. Instead, we use 100 random splits to evaluate the policies. Assuming no miscalibration ($\Gamma$ = 1), the resulting policies show valid coverage for all values of $\trt$ (\cref{fig:star-criteria}). In this scenario, the trade-off between controlling the treatment risk and minimizing the population risk still remains evident (\cref{fig:star-obj}). In \cref{fig:fft_star}, the learned policies for different tolerances $\trt$ are visualized. A low $\trt$ results in no assigned treatment, while a high $\trt$ leads to a treat-all policy. For $\trt=0.35$ and $\trt=0.40$ policies of depth two are identified and a portion of the population receives treatment.

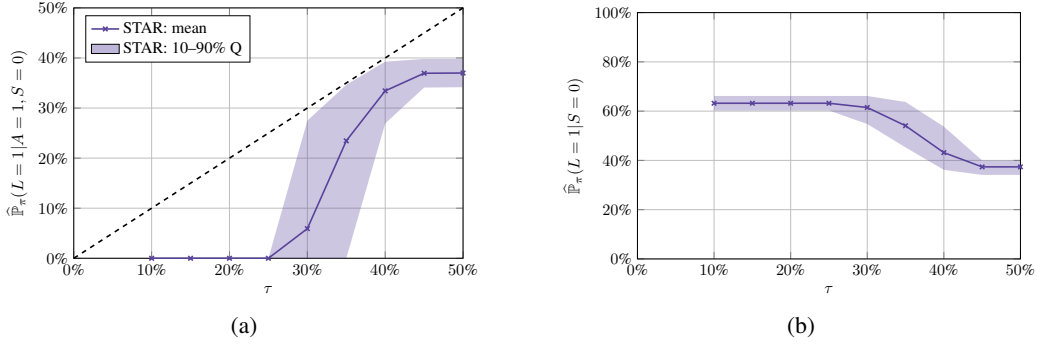
\begin{figure*}
\begin{subfigure}{0.47\textwidth}
    \centering
    \resizebox{\linewidth}{!}{ 
\begin{tikzpicture}
    \begin{axis}[
        xlabel={$\trt$},
        ylabel={$\widehat{\mathbb{P}}_{\pi}(L=1|A=1, S=0)$},
        yticklabel=\pgfmathparse{\tick*100}\pgfmathprintnumber{\pgfmathresult}\%,
        xticklabel=\pgfmathparse{\tick*100}\pgfmathprintnumber{\pgfmathresult}\%,
        xmin=0, xmax=0.5,
        ymin=0, ymax=0.5,
        legend pos=north west
    ]

    \addplot[color=Violet,  mark = x, line width=1pt] table[x=beta, y=mean constr, col sep=comma] {tikz/data/star.csv}; 
    \addlegendentry{STAR: mean}
    \addlegendimage{area legend, fill=Violet!30}
    \addlegendentry{STAR: 10–90\% Q}

    \addplot[name path=upper, draw=none] table[x=beta, y=constr_q90, col sep=comma] {tikz/data/star.csv};
    \addplot[name path=lower, draw=none] table[x=beta, y=constr_q10, col sep=comma] {tikz/data/star.csv};    
    \addplot[Violet, opacity=0.3] fill between[of=upper and lower];  

    \addplot [
        domain=0.0:0.5, 
        samples=10, 
        color=black,
        line width=1pt,
        dashed
        ]
        {x};

    \end{axis}
\end{tikzpicture}
    }
    \caption{}
    \label{fig:star-criteria}
\end{subfigure}
\hfill
\begin{subfigure}{0.47\textwidth}
    \centering
    \resizebox{\linewidth}{!}{ 
\begin{tikzpicture}
    \begin{axis}[
        xlabel={$\trt$},
        ylabel={$\widehat{\mathbb{P}}_{\pi}(L = 1|S=0)$},
        yticklabel=\pgfmathparse{\tick*100}\pgfmathprintnumber{\pgfmathresult}\%,
        xticklabel=\pgfmathparse{\tick*100}\pgfmathprintnumber{\pgfmathresult}\%,
        xmin=0, xmax=0.5,
        ymin=0, ymax=1,
        legend pos=north east
    ]

    \addplot[color=Violet,  mark = x, line width=1pt] table[x=beta, y=mean obj, col sep=comma] {tikz/data/star.csv}; 

    \addplot[name path=upper, draw=none] table[x=beta, y=obj_q90, col sep=comma] {tikz/data/star.csv};
    \addplot[name path=lower, draw=none] table[x=beta, y=obj_q10, col sep=comma] {tikz/data/star.csv};    
    \addplot[Violet, opacity=0.3] fill between[of=upper and lower];  

    \end{axis}
\end{tikzpicture}
    }
    \caption{}
    \label{fig:star-obj}
\end{subfigure}
\caption{Treatment and population risks of policies learned from STAR dataset. The shaded regions are obtained by randomized sample splitting.}
\label{fig:star}
\end{figure*}

\begin{figure*}
\begin{subfigure}{0.15\linewidth}
    \centering
    \begin{forest}
        [{Regular}, notreatnode]
    \end{forest}
    \vspace{20pt}
    \caption{$\trt \leq 30\%$}
\end{subfigure}
\hfill
\begin{subfigure}{0.3\linewidth}
    \centering
        \begin{forest}
            [{Free lunch}, root
                [{Teaching years \\ $< 6$}, root, EL={yes} 
                    [{Small}, treatnode, EL={yes}]
                    [{Regular}, notreatnode, EL={no}]]
                [{Small}, treatnode, EL={no}]
            ]
        \end{forest}
    \vspace{3pt}
    \caption{$\trt = 35\%$}
\end{subfigure}
\hfill
\begin{subfigure}{0.3\linewidth}
    \centering
        \begin{forest}
            [{Free lunch}, root
                [{Teaching years \\ $< 12$}, root, EL={yes} 
                    [{Small}, treatnode, EL={yes}]
                    [{Regular}, notreatnode, EL={no}]]
                [{Small}, treatnode, EL={no}]
            ]
        \end{forest}
    \vspace{3pt}
    \caption{$\trt = 40 \%$}
\end{subfigure}
\hfill
\begin{subfigure}{0.15\linewidth}
    \centering
        \begin{forest}
            [{Small}, treatnode]
        \end{forest}
    \vspace{20pt}
    \caption{$\trt \geq 45 \%$}
\end{subfigure}
\caption{Examples of learned treatment policies $\policy(\X; \trtparam_{\nrisk})$ using the STAR dataset, with student covariates $\X$ in grey, for different risk tolerances $\trt$.}
\label{fig:fft_star}
\end{figure*}
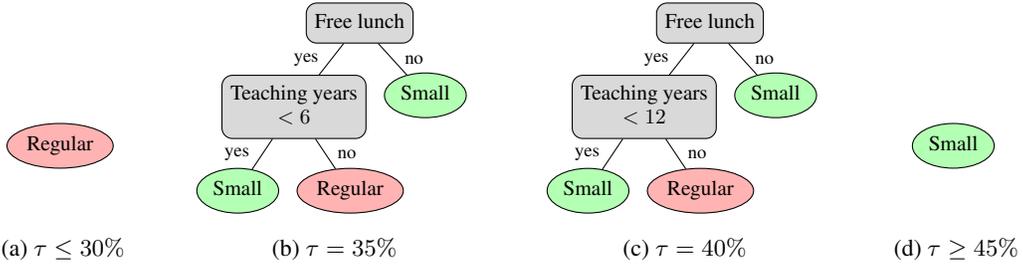

\section{Discussion}
\label{sec:discussion}
We introduced a learning method that advances trustworthiness in data-driven decision-making that seeks beneficial treatment allocations under explicit control of the treatment risk. The proposed approach is a certifiable method that ensures risk control in finite samples, even in partially identified settings. This makes it a valuable tool for high-stakes applications such as personalized medicine, economic policy, and safety-critical systems.

The proposed method provides a principled way to control treatment risk on average across the entire population. However, in settings where certain subgroups, defined by sensitive or high-stakes features such as age or socioeconomic status, exhibit significantly different risk profiles, aggregate control may be insufficient. In such cases, it may be necessary to stratify the population into appropriate subpopulations and apply risk control separately within each group to ensure fairness.


\begin{ack}
This work was partially supported by the Wallenberg AI, Autonomous Systems and
Software Program (WASP) funded by the Knut and Alice Wallenberg Foundation, and the Swedish Research Council under contract 2021-05022.
\end{ack}

\bibliography{ref}

\begin{thebibliography}{51}
\providecommand{\natexlab}[1]{#1}
\providecommand{\url}[1]{\texttt{#1}}
\expandafter\ifx\csname urlstyle\endcsname\relax
  \providecommand{\doi}[1]{doi: #1}\else
  \providecommand{\doi}{doi: \begingroup \urlstyle{rm}\Url}\fi

\bibitem[Achilles et~al.(2008)Achilles, Bain, Bellott, Boyd-Zaharias, Finn, Folger, Johnston, and Word]{DVN/SIWH9F_2008}
C.M. Achilles, Helen~Pate Bain, Fred Bellott, Jayne Boyd-Zaharias, Jeremy Finn, John Folger, John Johnston, and Elizabeth Word.
\newblock {Tennessee's Student Teacher Achievement Ratio (STAR) project}, 2008.
\newblock URL \url{https://doi.org/10.7910/DVN/SIWH9F}.

\bibitem[Adjaho and Christensen(2023)]{adjaho2023externally}
Christopher Adjaho and Timothy Christensen.
\newblock Externally valid policy choice.
\newblock \emph{arXiv preprint arXiv:2205.05561v3}, 1, 2023.

\bibitem[Angelopoulos et~al.(2022)Angelopoulos, Bates, Fisch, Lei, and Schuster]{angelopoulos2022conformal}
Anastasios~N Angelopoulos, Stephen Bates, Adam Fisch, Lihua Lei, and Tal Schuster.
\newblock Conformal risk control.
\newblock \emph{arXiv preprint arXiv:2208.02814}, 2022.

\bibitem[Athey and Imbens(2016)]{athey2016recursive}
Susan Athey and Guido Imbens.
\newblock Recursive partitioning for heterogeneous causal effects.
\newblock \emph{Proceedings of the National Academy of Sciences}, 113\penalty0 (27):\penalty0 7353--7360, 2016.

\bibitem[Athey and Wager(2021)]{athey2021policy}
Susan Athey and Stefan Wager.
\newblock Policy learning with observational data.
\newblock \emph{Econometrica}, 89\penalty0 (1):\penalty0 133--161, 2021.

\bibitem[Bates et~al.(2021)Bates, Angelopoulos, Lei, Malik, and Jordan]{bates2021distribution}
Stephen Bates, Anastasios Angelopoulos, Lihua Lei, Jitendra Malik, and Michael Jordan.
\newblock Distribution-free, risk-controlling prediction sets.
\newblock \emph{Journal of the ACM (JACM)}, 68\penalty0 (6):\penalty0 1--34, 2021.

\bibitem[Ben-Michael et~al.(2025)Ben-Michael, Greiner, Imai, and Jiang]{ben2025safe}
Eli Ben-Michael, D~James Greiner, Kosuke Imai, and Zhichao Jiang.
\newblock Safe policy learning through extrapolation: Application to pre-trial risk assessment.
\newblock \emph{Journal of the American Statistical Association}, pages 1--23, 2025.

\bibitem[Caruana et~al.(2015)Caruana, Lou, Gehrke, Koch, Sturm, and Elhadad]{caruana2015intelligible}
Rich Caruana, Yin Lou, Johannes Gehrke, Paul Koch, Marc Sturm, and Noemie Elhadad.
\newblock Intelligible models for healthcare: Predicting pneumonia risk and hospital 30-day readmission.
\newblock In \emph{Proceedings of the 21th ACM SIGKDD international conference on knowledge discovery and data mining}, pages 1721--1730, 2015.

\bibitem[Christensen et~al.(2023)Christensen, Moon, and Schorfheide]{christensen2023optimal}
Timothy Christensen, Hyungsik~Roger Moon, and Frank Schorfheide.
\newblock Optimal discrete decisions when payoffs are partially identified.
\newblock \emph{arXiv preprint arXiv:2204.11748v2}, 2023.

\bibitem[Cui(2021)]{cui2021individualized}
Yifan Cui.
\newblock Individualized decision-making under partial identification: Three perspectives, two optimality results, and one paradox.
\newblock \emph{Harvard Data Science Review}, 3\penalty0 (3), 2021.

\bibitem[Doubleday et~al.(2022)Doubleday, Zhou, Zhou, and Fu]{doubleday2022risk}
Kevin Doubleday, Jin Zhou, Hua Zhou, and Haoda Fu.
\newblock Risk controlled decision trees and random forests for precision medicine.
\newblock \emph{Statistics in medicine}, 41\penalty0 (4):\penalty0 719--735, 2022.

\bibitem[Dud{\'\i}k et~al.(2011)Dud{\'\i}k, Langford, and Li]{dudik2011doubly}
Miroslav Dud{\'\i}k, John Langford, and Lihong Li.
\newblock Doubly robust policy evaluation and learning.
\newblock In \emph{Proceedings of the 28th International Conference on International Conference on Machine Learning}, pages 1097--1104, 2011.

\bibitem[Ek and Zachariah(2024)]{ek2024externally}
Sofia Ek and Dave Zachariah.
\newblock Externally valid policy evaluation from randomized trials using additional observational data.
\newblock In \emph{The Thirty-eighth Annual Conference on Neural Information Processing Systems}, 2024.
\newblock URL \url{https://openreview.net/forum?id=2pgc5xDJ1b}.

\bibitem[Gigerenzer et~al.(2000)Gigerenzer, Todd, ABC Research~Group, et~al.]{gigerenzer2000simple}
Gerd Gigerenzer, Peter~M Todd, the ABC Research~Group, et~al.
\newblock \emph{Simple heuristics that make us smart}.
\newblock Oxford University Press, 2000.

\bibitem[Group et~al.(1997)]{international1997international}
International Stroke Trial~Collaborative Group et~al.
\newblock The international stroke trial (ist): a randomised trial of aspirin, subcutaneous heparin, both, or neither among 19 435 patients with acute ischaemic stroke.
\newblock \emph{The Lancet}, 349\penalty0 (9065):\penalty0 1569--1581, 1997.

\bibitem[Hoogland et~al.(2021)Hoogland, IntHout, Belias, Rovers, Riley, E.~Harrell~Jr, Moons, Debray, and Reitsma]{hoogland2021tutorial}
Jeroen Hoogland, Joanna IntHout, Michail Belias, Maroeska~M Rovers, Richard~D Riley, Frank E.~Harrell~Jr, Karel~GM Moons, Thomas~PA Debray, and Johannes~B Reitsma.
\newblock A tutorial on individualized treatment effect prediction from randomized trials with a binary endpoint.
\newblock \emph{Statistics in medicine}, 40\penalty0 (26):\penalty0 5961--5981, 2021.

\bibitem[Huang et~al.(2021)Huang, Leqi, Lipton, and Azizzadenesheli]{huang2021off}
Audrey Huang, Liu Leqi, Zachary Lipton, and Kamyar Azizzadenesheli.
\newblock Off-policy risk assessment in contextual bandits.
\newblock \emph{Advances in Neural Information Processing Systems}, 34:\penalty0 23714--23726, 2021.

\bibitem[Huang(2024)]{huang2024sensitivity}
Melody~Y Huang.
\newblock Sensitivity analysis for the generalization of experimental results.
\newblock \emph{Journal of the Royal Statistical Society Series A: Statistics in Society}, page qnae012, 03 2024.

\bibitem[Ichino et~al.(2008)Ichino, Mealli, and Nannicini]{ichino2008temporary}
Andrea Ichino, Fabrizia Mealli, and Tommaso Nannicini.
\newblock From temporary help jobs to permanent employment: what can we learn from matching estimators and their sensitivity?
\newblock \emph{Journal of applied econometrics}, 23\penalty0 (3):\penalty0 305--327, 2008.

\bibitem[Imbens and Rubin(2015)]{imbens2015causal}
Guido~W Imbens and Donald~B Rubin.
\newblock \emph{Causal inference in statistics, social, and biomedical sciences}.
\newblock Cambridge university press, 2015.

\bibitem[Kallus(2018)]{kallus2018balanced}
Nathan Kallus.
\newblock Balanced policy evaluation and learning.
\newblock \emph{Advances in neural information processing systems}, 31, 2018.

\bibitem[Kallus(2022)]{kallus2022s}
Nathan Kallus.
\newblock What's the harm? sharp bounds on the fraction negatively affected by treatment.
\newblock \emph{Advances in Neural Information Processing Systems}, 35:\penalty0 15996--16009, 2022.

\bibitem[Kallus and Zhou(2021)]{kallus2021minimax}
Nathan Kallus and Angela Zhou.
\newblock Minimax-optimal policy learning under unobserved confounding.
\newblock \emph{Management Science}, 67\penalty0 (5):\penalty0 2870--2890, 2021.

\bibitem[Kallus et~al.(2018)Kallus, Puli, and Shalit]{kallus2018removing}
Nathan Kallus, Aahlad~Manas Puli, and Uri Shalit.
\newblock Removing hidden confounding by experimental grounding.
\newblock \emph{Advances in neural information processing systems}, 31, 2018.

\bibitem[Katsikopoulos et~al.(2021)Katsikopoulos, Simsek, Buckmann, and Gigerenzer]{katsikopoulos2021classification}
Konstantinos~V Katsikopoulos, Ozgur Simsek, Marcus Buckmann, and Gerd Gigerenzer.
\newblock \emph{Classification in the wild: The science and art of transparent decision making}.
\newblock MIT Press, 2021.

\bibitem[Kitagawa and Tetenov(2018)]{kitagawa2018should}
Toru Kitagawa and Aleksey Tetenov.
\newblock Who should be treated? empirical welfare maximization methods for treatment choice.
\newblock \emph{Econometrica}, 86\penalty0 (2):\penalty0 591--616, 2018.

\bibitem[Krueger(1999)]{krueger1999experimental}
Alan~B Krueger.
\newblock Experimental estimates of education production functions.
\newblock \emph{The quarterly journal of economics}, 114\penalty0 (2):\penalty0 497--532, 1999.

\bibitem[Lei et~al.(2018)Lei, G’Sell, Rinaldo, Tibshirani, and Wasserman]{lei2018distribution}
Jing Lei, Max G’Sell, Alessandro Rinaldo, Ryan~J Tibshirani, and Larry Wasserman.
\newblock Distribution-free predictive inference for regression.
\newblock \emph{Journal of the American Statistical Association}, 113\penalty0 (523):\penalty0 1094--1111, 2018.

\bibitem[Li et~al.(2023)Li, Zheng, Cao, Geng, Liu, and Wu]{li2023trustworthy}
Haoxuan Li, Chunyuan Zheng, Yixiao Cao, Zhi Geng, Yue Liu, and Peng Wu.
\newblock Trustworthy policy learning under the counterfactual no-harm criterion.
\newblock In \emph{International Conference on Machine Learning}, pages 20575--20598. PMLR, 2023.

\bibitem[Manski(2003)]{manski2003identification}
Charles~F Manski.
\newblock Identification problems in the social sciences and everyday life.
\newblock \emph{Southern Economic Journal}, 70\penalty0 (1):\penalty0 11--21, 2003.

\bibitem[Manski(2004)]{manski2004statistical}
Charles~F Manski.
\newblock Statistical treatment rules for heterogeneous populations.
\newblock \emph{Econometrica}, 72\penalty0 (4):\penalty0 1221--1246, 2004.

\bibitem[Manski(2007)]{manski2007identification}
Charles~F Manski.
\newblock \emph{Identification for prediction and decision}.
\newblock Harvard University Press, 2007.

\bibitem[Nie and Wager(2021)]{nie2021quasi}
Xinkun Nie and Stefan Wager.
\newblock Quasi-oracle estimation of heterogeneous treatment effects.
\newblock \emph{Biometrika}, 108\penalty0 (2):\penalty0 299--319, 2021.

\bibitem[Peters et~al.(2017)Peters, Janzing, and Sch{\"o}lkopf]{peters2017elements}
Jonas Peters, Dominik Janzing, and Bernhard Sch{\"o}lkopf.
\newblock \emph{Elements of causal inference: foundations and learning algorithms}.
\newblock The MIT Press, 2017.

\bibitem[Qian and Murphy(2011)]{qian2011performance}
Min Qian and Susan~A Murphy.
\newblock Performance guarantees for individualized treatment rules.
\newblock \emph{Annals of statistics}, 39\penalty0 (2):\penalty0 1180, 2011.

\bibitem[Rudin(2019)]{rudin2019stop}
Cynthia Rudin.
\newblock Stop explaining black box machine learning models for high stakes decisions and use interpretable models instead.
\newblock \emph{Nature machine intelligence}, 1\penalty0 (5):\penalty0 206--215, 2019.

\bibitem[Sandercock et~al.(2011)Sandercock, Niewada, Cz{\l}onkowska, and Group]{sandercock2011international}
Peter~AG Sandercock, Maciej Niewada, Anna Cz{\l}onkowska, and International Stroke Trial~Collaborative Group.
\newblock The international stroke trial database.
\newblock \emph{Trials}, 12\penalty0 (1):\penalty0 101, 2011.

\bibitem[Sarvet and Stensrud(2023)]{sarvet2023perspective}
Aaron~L Sarvet and Mats~J Stensrud.
\newblock Perspective on ‘harm’in personalized medicine.
\newblock \emph{American Journal of Epidemiology}, page kwad162, 2023.

\bibitem[Smith(2005)]{smith2005origin}
Cedric~M Smith.
\newblock Origin and uses of primum non nocere—above all, do no harm!
\newblock \emph{The Journal of Clinical Pharmacology}, 45\penalty0 (4):\penalty0 371--377, 2005.

\bibitem[Swaminathan and Joachims(2015)]{swaminathan2015counterfactual}
Adith Swaminathan and Thorsten Joachims.
\newblock Counterfactual risk minimization: Learning from logged bandit feedback.
\newblock In \emph{International Conference on Machine Learning}, pages 814--823. PMLR, 2015.

\bibitem[Tan(2006)]{tan2006distributional}
Zhiqiang Tan.
\newblock A distributional approach for causal inference using propensity scores.
\newblock \emph{Journal of the American Statistical Association}, 101\penalty0 (476):\penalty0 1619--1637, 2006.

\bibitem[Vickers and Elkin(2006)]{vickers2006decision}
Andrew~J Vickers and Elena~B Elkin.
\newblock Decision curve analysis: a novel method for evaluating prediction models.
\newblock \emph{Medical Decision Making}, 26\penalty0 (6):\penalty0 565--574, 2006.

\bibitem[Vovk et~al.(2005)Vovk, Gammerman, and Shafer]{vovk2005algorithmic}
Vladimir Vovk, Alexander Gammerman, and Glenn Shafer.
\newblock \emph{Algorithmic learning in a random world}.
\newblock Springer Science \& Business Media, 2005.

\bibitem[Wager and Athey(2018)]{wager2018estimation}
Stefan Wager and Susan Athey.
\newblock Estimation and inference of heterogeneous treatment effects using random forests.
\newblock \emph{Journal of the American Statistical Association}, 113\penalty0 (523):\penalty0 1228--1242, 2018.

\bibitem[Wang et~al.(2018)Wang, Fu, and Zeng]{wang2018learning}
Yuanjia Wang, Haoda Fu, and Donglin Zeng.
\newblock Learning optimal personalized treatment rules in consideration of benefit and risk: with an application to treating type 2 diabetes patients with insulin therapies.
\newblock \emph{Journal of the American Statistical Association}, 113\penalty0 (521):\penalty0 1--13, 2018.

\bibitem[Wasserman(2013)]{wasserman2013all}
Larry Wasserman.
\newblock \emph{All of statistics: a concise course in statistical inference}.
\newblock Springer Science \& Business Media, 2013.

\bibitem[Westreich(2019)]{westreich2019epidemiology}
Daniel Westreich.
\newblock \emph{Epidemiology by Design: A Causal Approach to the Health Sciences}.
\newblock Oxford University Press, Incorporated, 2019.
\newblock ISBN 9780190665760.

\bibitem[Yata(2025)]{yata2025optimal}
Kohei Yata.
\newblock Optimal decision rules under partial identification.
\newblock \emph{arXiv preprint arXiv:2111.04926v4}, 2025.

\bibitem[Zhang et~al.(2012)Zhang, Tsiatis, Laber, and Davidian]{zhang2012robust}
Baqun Zhang, Anastasios~A Tsiatis, Eric~B Laber, and Marie Davidian.
\newblock A robust method for estimating optimal treatment regimes.
\newblock \emph{Biometrics}, 68\penalty0 (4):\penalty0 1010--1018, 2012.

\bibitem[Zhang et~al.(2015)Zhang, Laber, Tsiatis, and Davidian]{zhang2015using}
Yichi Zhang, Eric~B Laber, Anastasios Tsiatis, and Marie Davidian.
\newblock Using decision lists to construct interpretable and parsimonious treatment regimes.
\newblock \emph{Biometrics}, 71\penalty0 (4):\penalty0 895--904, 2015.

\bibitem[Zhao et~al.(2012)Zhao, Zeng, Rush, and Kosorok]{zhao2012estimating}
Yingqi Zhao, Donglin Zeng, A~John Rush, and Michael~R Kosorok.
\newblock Estimating individualized treatment rules using outcome weighted learning.
\newblock \emph{Journal of the American Statistical Association}, 107\penalty0 (499):\penalty0 1106--1118, 2012.

\end{thebibliography}
\bibliographystyle{plainnat}


\newpage
\appendix
\section*{Appendix}
In \cref{sec:experiments_app} we provide additional details on the numerical experiments discussed in Section~\ref{sec:experiments} and in \cref{sec:app_alt_method} we present a method that yields an alternative, weaker control on the treatment risk.

\section{Additional Details Experiments}
\label{sec:experiments_app}
All experiments were carried out on a laptop using only the CPU. The synthetic experiments required approximately two and a half to three hours for 1000 runs, while the real-data experiments took 10 to 15 minutes for 100 runs.

\subsection{Fast-and-Frugal Policy Learning}
\label{sec:fft}
Fast-and-frugal trees (FFTs) are rule-based decision trees designed for high-stakes, time-sensitive environments such as medical diagnostics and emergency response \citep{gigerenzer2000simple}. Their structure supports early decision-making, enabling some cases to be resolved without using all available information. For example, in medical testing, certain tests may be unnecessary for some patients, helping to optimize both time and resources \citep{katsikopoulos2021classification}.

We use FFTs to construct the policies in (\ref{eq:minimization_policy}), following the greedy approach of \citet{zhang2015using}. We restrict our rules to one variable in each condition and grid continuous covariates. The algorithm evaluates
\begin{equation}
    \Ehat_{\npolicy} \left[L \cdot \overline{\Weight}^{\Gamma}\right],
\end{equation}
for each potential split and selects the one that minimizes this value. The key distinction from \citet{zhang2015using} is that we additionally evaluate
\begin{equation}
    \Ehat_{\npolicy}\left[L \cdot \frac{\ind{A =1}}{\probpolicy(\Dec=1|\Sampling=0)} \cdot \overline{\Weight}^{\Gamma}\right] \leq \trtparam,
\end{equation}
at each greedy step and impose a constraint that only permits splits if this criterion is met. Before a split is performed, the current tree is stored. At each step, two new candidate trees are generated: one that continues to grow if the criterion is satisfied, and another that extends in the alternative direction. This process iterates until a stopping criterion is reached, either when the maximum depth is attained or when further splits fail to improve the objective function. Once all trees have been constructed, the final model is selected as the tree that minimizes the objective function.

\subsection{Benchmarking Degree of Miscalibration}
\label{sec:app_benchmarking}

We now demonstrate how a credible range for $\Gamma$ can be benchmarked, using an approach from \citet{huang2021off} and \citet{ek2024externally}. To this end, we estimate the propensity score via logistic regression for the STAR dataset for illustration. As a first step, we assess whether this model offers sufficient flexibility for the task. The nominal assignment odds in (\ref{eq:Gamma_bound_a}) are discretized into five bins. Within each bin, the unknown assignment odds are estimated by computing the empirical ratio of the samples for which $\Dec = 1$ and $\Dec = 0$, respectively. If the model is sufficiently flexible, the estimated assignment odds should approximate the nominal odds within each bin. This is observed in \cref{fig:calibration_curve}.

We continue with a benchmarking approach designed to account for the potential influence of unobserved individual factors $\Unobserved$, by treating observed covariates in $\X$ as unmeasured. Let the $k$th covariate $\X_k$ correspond to an omitted factor, while $\X_{-k}$ denotes all the remaining observed covariates and define
\begin{equation}
    \oddshat(\X_{-k}, \X_k) = \frac{1 - \widehat{\prob}(\Dec | \X)}{\widehat{\prob}(\Dec | \X)}
    , \quad  
    \oddshat(\X_{-k}) = \frac{1 - \widehat{\prob}(\Dec |  \X_{-k})}{\widehat{\prob}(\Dec | \X_{-k})}.
\end{equation}

The ratio $\oddshat(\X_{-k}, \X_k) / \oddshat(\X_{-k})$ is then used to benchmark odds ratio in (\ref{eq:Gamma_bound_a}). \Cref{fig:calibration_missing} illustrates this for the two most influential observed covariates. If the unobserved individual factor $\Unobserved$ has an effect on the propensity score that is no greater than than these, then a credible range for $\Gamma$ is between 1.5 and 1.7.

\begin{figure*}
\begin{subfigure}{0.47\textwidth}
    \centering
    \resizebox{\linewidth}{!}{ 
        \begin{tikzpicture}
    \begin{axis}[
        xlabel={Nominal odds},
        ylabel={Estimated odds},
        xmin=0, xmax=1.5,
        ymin=0, ymax=1.5,
        legend pos=north west
    ]

    \addplot[color=RoyalBlue, mark = o, line width=1pt] table[x=odds_pred, y=odds_true, col sep=comma] {tikz/data/calibration_curve.csv}; 
    \addlegendentry{Logistic}

    \addplot [
        domain=0.0:1.5, 
        samples=10, 
        color=black,
        line width=1pt,
        dashed
        ]
        {x}; 
    \addlegendentry{$\Gamma = 1$}

    \addplot [
        domain=0.0:1.5, 
        samples=10, 
        color=gray,
        line width=1pt,
        dashed
        ]
        {1.5 * x};
    \addlegendentry{$\Gamma = 1.5$}

    \addplot [
        domain=0.0:1.5, 
        samples=10, 
        color=gray,
        line width=1pt,
        dashed
        ]
        {1 / 1.5 * x};
    \end{axis};
    
\end{tikzpicture}
    }
    \caption{}
    \label{fig:calibration_curve}
\end{subfigure}  
\hfill
\begin{subfigure}{0.47\textwidth}
    \centering
    \resizebox{\linewidth}{!}{ 
        \begin{tikzpicture}
    \begin{axis}[
        xlabel={$\oddshat(\X_{-k}, \X_k)/\oddshat(\X_{-k})$},
        ylabel={},
        yticklabel=\pgfmathparse{\tick*100}\pgfmathprintnumber{\pgfmathresult}\%,
        xmin=0.5, xmax=1.8,
        ymin=0, ymax=1,
        legend pos=north west
    ]

    \addplot[color=RoyalBlue, line width=1pt] table[x=g1tcareer, y=ecdf_y, col sep=comma] {tikz/data/missing_odds.csv}; 
    \addlegendentry{g1tcareer}

    \addplot[color=Peach, line width=1pt] table[x=g1tyears, y=ecdf_y, col sep=comma] {tikz/data/missing_odds.csv}; 
    \addlegendentry{g1tyears}

    \addplot [
        domain=0.0:1.5, 
        samples=10, 
        color=black,
        line width=1pt,
        dashed
        ]
        ({1.0}, x); 
    \addlegendentry{$\Gamma = 1$}

    \addplot [
        domain=0.0:1.5, 
        samples=10, 
        color=gray,
        line width=1pt,
        dashed
        ]
        ({1.5}, x); 
    \addlegendentry{$\Gamma = 1.5$}

    \addplot [
        domain=0.0:1.5, 
        samples=10, 
        color=gray,
        line width=1pt,
        dashed
        ]
        ({1/1.5}, x); 
    \end{axis};
    
\end{tikzpicture}
    }
    \caption{}
    \label{fig:calibration_missing}
\end{subfigure}
\caption{Benchmarking $\Gamma$. (a) Reliability diagram of the observed odds against the average predicted nominal odds. (b) Benchmarking $\Gamma$ using omitted covariates.}
\end{figure*}
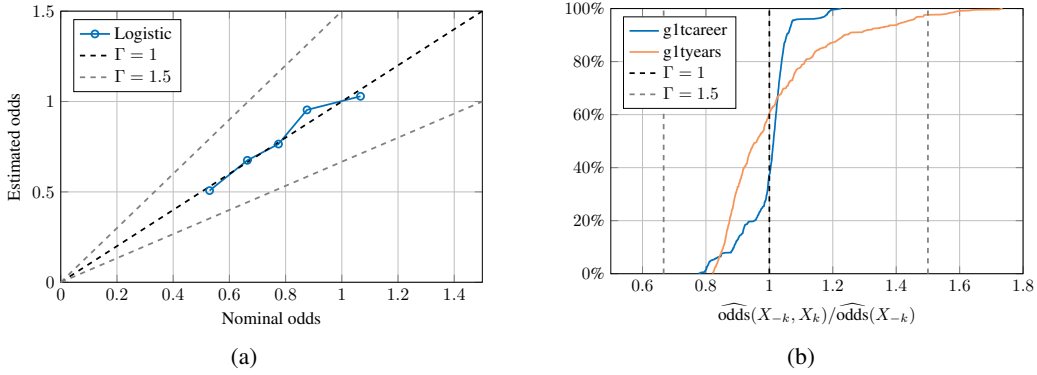

\subsection{Synthetic Data}
\label{sec:app_synthetic}
For completeness, we add $\Gamma=2$ to the experiments in \cref{fig:motivating-example}. Here, the treatment assignments are known and $\Gamma = 1$ is valid. Using $\Gamma = 2$ makes the algorithm more conservative, but, more importantly, the guarantee is valid; see \cref{fig:app_synthetic-gamma}. 

\begin{figure*}
\begin{subfigure}{0.47\textwidth}
    \centering
    \resizebox{\linewidth}{!}{ 
        \begin{tikzpicture}
    \begin{axis}[
        xlabel={$\trt$},
        ylabel={$\mathbb{P}_{\pi}(L = 1|A=1, S=0)$},
        yticklabel=\pgfmathparse{\tick*100}\pgfmathprintnumber{\pgfmathresult}\%,
        xticklabel=\pgfmathparse{\tick*100}\pgfmathprintnumber{\pgfmathresult}\%,
        xmin=0, xmax=0.4,
        ymin=0, ymax=0.4,
        legend pos=north west
    ]

    \addplot[color=RoyalBlue, mark = o, line width=1pt] table[x=beta, y=mean constr, col sep=comma] {tikz/data/syn_true1.csv}; 
    \addlegendentry{$\Gamma = 1$}

    \addplot[color=Peach, mark = star, line width=1pt] table[x=beta, y=mean constr, col sep=comma] {tikz/data/syn_true2.csv}; 
    \addlegendentry{$\Gamma = 2$}

    \addplot[name path=upper, draw=none] table[x=beta, y=constr_q90, col sep=comma] {tikz/data/syn_true1.csv};
    \addplot[name path=lower, draw=none] table[x=beta, y=constr_q10, col sep=comma] {tikz/data/syn_true1.csv};    
    \addplot[RoyalBlue, opacity=0.3] fill between[of=upper and lower];  

    \addplot[name path=upper, draw=none] table[x=beta, y=constr_q90, col sep=comma] {tikz/data/syn_true2.csv};
    \addplot[name path=lower, draw=none] table[x=beta, y=constr_q10, col sep=comma] {tikz/data/syn_true2.csv};    
    \addplot[Peach, opacity=0.3] fill between[of=upper and lower];  

    \addplot [
        domain=0.0:0.4, 
        samples=10, 
        color=black,
        line width=1pt,
        dashed
        ]
        {x};
    \end{axis}
\end{tikzpicture}
    }
    \caption{}
    \label{fig:app_synthetic-gamma-criteria}
\end{subfigure}  
\hfill
\begin{subfigure}{0.47\textwidth}
    \centering
    \resizebox{\linewidth}{!}{ 
        \begin{tikzpicture}
    \begin{axis}[
        xlabel={$\trt$},
        ylabel={$\mathbb{P}_{\pi}(L = 1|S=0)$},
        yticklabel=\pgfmathparse{\tick*100}\pgfmathprintnumber{\pgfmathresult}\%,
        xticklabel=\pgfmathparse{\tick*100}\pgfmathprintnumber{\pgfmathresult}\%,
        xmin=0, xmax=0.4,
        ymin=0, ymax=1.0,
        legend pos=north east
    ]

    \addplot[color=RoyalBlue, mark = o, line width=1pt] table[x=beta, y=mean obj, col sep=comma] {tikz/data/syn_true1.csv}; 
    \addlegendentry{$\Gamma = 1$}

    \addplot[color=Peach, mark = star, line width=1pt] table[x=beta, y=mean obj, col sep=comma] {tikz/data/syn_true2.csv}; 
    \addlegendentry{$\Gamma = 2$}

    \addplot[name path=upper, draw=none] table[x=beta, y=obj_q90, col sep=comma] {tikz/data/syn_true1.csv};
    \addplot[name path=lower, draw=none] table[x=beta, y=obj_q10, col sep=comma] {tikz/data/syn_true1.csv};    
    \addplot[RoyalBlue, opacity=0.3] fill between[of=upper and lower];  

    \addplot[name path=upper, draw=none] table[x=beta, y=obj_q90, col sep=comma] {tikz/data/syn_true2.csv};
    \addplot[name path=lower, draw=none] table[x=beta, y=obj_q10, col sep=comma] {tikz/data/syn_true2.csv};    
    \addplot[Peach, opacity=0.3] fill between[of=upper and lower];  

    \end{axis}
\end{tikzpicture}
    }
    \caption{}
    \label{fig:app_synthetic-gamma-obj}
\end{subfigure}
\caption{The treatment risk under policy $\policy$, and the population risk under the same policy for different values of the treatment risk tolerance $\trt$ and different confounding assumptions ($\Gamma = {1 ,2}$). (a) Treatment risk falls below any given $\trt$ with a probability of at least $90\%$. (b) The trade-off between population risk and treatment risk tolerance $\trt$.}
\label{fig:app_synthetic-gamma}
\end{figure*}
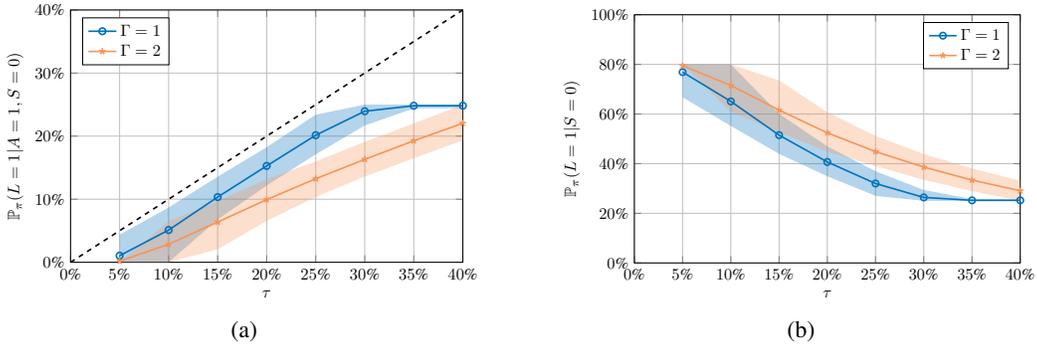

We also extend the synthetic experiments with experiments for \rct{}, similarly to the confounding case in the main paper. The covariates $\X$ are two-dimensional and $\prob(\X)$ is given by
\begin{equation}
    \X = \begin{bmatrix} \X_1 \\ \X_2 \end{bmatrix}|\Sampling=1 \sim \mathcal{U}(30,80)^2.
\end{equation} 
The assignment of treatment actions $\Dec \in \{0, 1\}$ is $\prob(\Dec=1) = 0.5$ and is known. We also have a confounding variable $\Unobserved$, uniformly distributed $\Unobserved \sim \mathcal{U}(0,1)$, that affects the sampling $S$ and the loss $\Loss$. The sampling distribution is given by
\begin{equation*}
\begin{aligned}
    \prob(\Sampling=1 |\X, \Unobserved) = \ind{\Unobserved < 0.5} \cdot \frac{\sigma_{\text{nom}}(\X)}{2 - \sigma_{\text{nom}}(\X)} 
     + \ind{\Unobserved \geq 0.5} \cdot \frac{\sigma_{\text{nom}}(\X)}{0.5 + 0.5\sigma_{\text{nom}}(\X)},
\end{aligned}
\end{equation*} 
where $\sigma_{\text{nom}}(\X)$ is the propensity score in  (\ref{eq:synthetic_propensity}). This corresponds to an error of $\Gamma = 2$ in (\ref{eq:Gamma_bound_s}). The loss for the treated is defined as in (\ref{eq:syn_loss_confounding}). The evaluation follows the procedure described in \cref{sec:exp_synthetic}, with the key difference that, instead of relying on a closed-form expression, each Monte Carlo run is now evaluated using 18,000 additional samples drawn from the test distribution.

The decision policy is learned for both $\Gamma = {1, 2}$. In \cref{fig:app_synthetic-rct-criteria}, we observe that for $\Gamma = 1$, the guarantee that the treatment risk remains below $\trt$ in at least $1-\miscoveragerate$ of the trials is violated. In contrast, $\Gamma = 2$ provides a valid model for the selection odds, ensuring that the treatment risk of the resulting policies stays below $\trt$ in the guaranteed proportion of trials. As shown in \cref{fig:app_synthetic-rct-obj}, the trade-off between minimizing the probability of nonrecovery and controlling treatment risk remains evident: using $\Gamma = 2$ improves treatment risk control but leads to a higher probability of nonrecovery.

\begin{figure*}
\begin{subfigure}{0.47\textwidth}
    \centering
    \resizebox{\linewidth}{!}{ 
        \begin{tikzpicture}
    \begin{axis}[
        xlabel={$\trt$},
        ylabel={$\mathbb{P}_{\pi}(L = 1|A=1, S=0)$},
        yticklabel=\pgfmathparse{\tick*100}\pgfmathprintnumber{\pgfmathresult}\%,
        xticklabel=\pgfmathparse{\tick*100}\pgfmathprintnumber{\pgfmathresult}\%,
        xmin=0, xmax=0.4,
        ymin=0, ymax=0.4,
        legend pos=north west
    ]

    \addplot[color=RoyalBlue, mark = o, line width=1pt] table[x=beta, y=mean constr, col sep=comma] {tikz/data/syn_rct_gamma1.csv}; 
    \addlegendentry{$\Gamma = 1$}

    \addplot[color=Peach, mark = star, line width=1pt] table[x=beta, y=mean constr, col sep=comma] {tikz/data/syn_rct_gamma2.csv}; 
    \addlegendentry{$\Gamma = 2$}

    \addplot[name path=upper, draw=none] table[x=beta, y=constr_q90, col sep=comma] {tikz/data/syn_rct_gamma1.csv};
    \addplot[name path=lower, draw=none] table[x=beta, y=constr_q10, col sep=comma] {tikz/data/syn_rct_gamma1.csv};    
    \addplot[RoyalBlue, opacity=0.3] fill between[of=upper and lower];  

    \addplot[name path=upper, draw=none] table[x=beta, y=constr_q90, col sep=comma] {tikz/data/syn_rct_gamma2.csv};
    \addplot[name path=lower, draw=none] table[x=beta, y=constr_q10, col sep=comma] {tikz/data/syn_rct_gamma2.csv};    
    \addplot[Peach, opacity=0.3] fill between[of=upper and lower];  

    \addplot [
        domain=0.0:0.4, 
        samples=10, 
        color=black,
        line width=1pt,
        dashed
        ]
        {x};
    \end{axis}
\end{tikzpicture}
    }
    \caption{}
    \label{fig:app_synthetic-rct-criteria}
\end{subfigure}  
\hfill
\begin{subfigure}{0.47\textwidth}
    \centering
    \resizebox{\linewidth}{!}{ 
        \begin{tikzpicture}
    \begin{axis}[
        xlabel={$\trt$},
        ylabel={$\mathbb{P}_{\pi}(L = 1|S=0)$},
        yticklabel=\pgfmathparse{\tick*100}\pgfmathprintnumber{\pgfmathresult}\%,
        xticklabel=\pgfmathparse{\tick*100}\pgfmathprintnumber{\pgfmathresult}\%,
        xmin=0, xmax=0.4,
        ymin=0, ymax=1.0,
        legend pos=north east
    ]

    \addplot[color=RoyalBlue, mark = o, line width=1pt] table[x=beta, y=mean obj, col sep=comma] {tikz/data/syn_rct_gamma1.csv}; 
    

    \addplot[color=Peach, mark = star, line width=1pt] table[x=beta, y=mean obj, col sep=comma] {tikz/data/syn_rct_gamma2.csv}; 
    

    \addplot[name path=upper, draw=none] table[x=beta, y=obj_q90, col sep=comma] {tikz/data/syn_rct_gamma1.csv};
    \addplot[name path=lower, draw=none] table[x=beta, y=obj_q10, col sep=comma] {tikz/data/syn_rct_gamma1.csv};    
    \addplot[RoyalBlue, opacity=0.3] fill between[of=upper and lower];  

    \addplot[name path=upper, draw=none] table[x=beta, y=obj_q90, col sep=comma] {tikz/data/syn_rct_gamma2.csv};
    \addplot[name path=lower, draw=none] table[x=beta, y=obj_q10, col sep=comma] {tikz/data/syn_rct_gamma2.csv};    
    \addplot[Peach, opacity=0.3] fill between[of=upper and lower];  

    \end{axis}
\end{tikzpicture}
    }
    \caption{}
    \label{fig:app_synthetic-rct-obj}
\end{subfigure}
\caption{The treatment risk under policy $\policy$, and the population risk under the same policy for different values of the treatment risk tolerance $\trt$ and different confounding assumptions ($\Gamma = {1 ,2}$). (a) For $\Gamma = 1$, the coverage guarantee that the treatment risk should be below $\tau$ with a high probability of at least $1 - \miscoveragerate$ does not hold across a range of $\trt$ values, while $\Gamma = 2$ provides a valid model for the degree of miscalibration, ensuring that the resulting policies remain valid across all $\trt$ values. (b) The trade-off between the probability of nonrecovery and different choices of $\trt$.}
\label{fig:app_synthetic-rct}
\end{figure*}
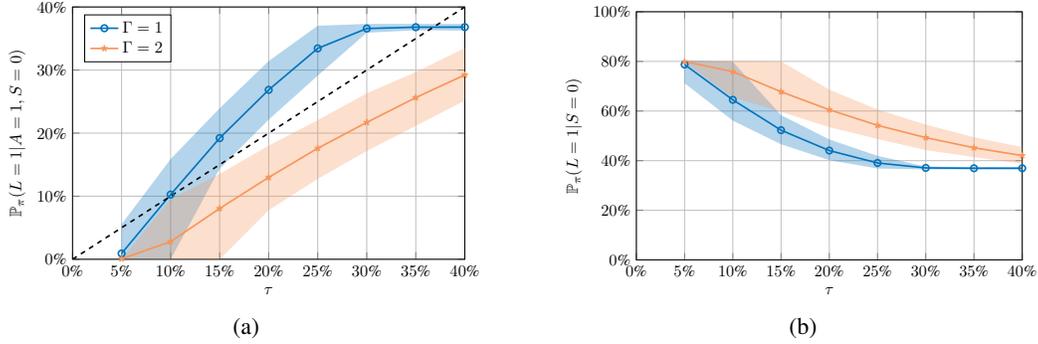

\subsection{STAR Data}
\label{sec:app_star}
For the experiments in \cref{sec:star} the 11 covariates in \cref{tab:star} are used.

\begin{table}[htbp]
\caption{The covariates used in the STAR experiments.}
\begin{center}
\begin{tabular}{l|l|c}
\toprule
Covariate & Description & Categorical \\ 
\midrule
gender & Student gender & yes \\
race & Ethnicity & yes \\
g1promote & Promoted from grade 1 & yes \\
g1specin & Special instruction & yes \\
g1surban & School location & yes \\
g1freelunch & Free/reduced lunch & yes \\
birthmonth & Birth month & no \\ 
g1present & School days present & no \\
g1absent & School days absent & no \\ 
g1tcareer & Teacher’s career level & no \\
g1tyears & Teaching experience (years) & no \\  
\bottomrule
\end{tabular}
\end{center}
\label{tab:star}
\end{table}

\subsection{International Stroke Trial Data}
\label{sec:app_stroke}
The International Stroke Trial (IST) was a large randomized trial that evaluated the effects of aspirin and heparin in acute ischemic stroke. The original trial included 19,435 patients and compared four arms: aspirin, heparin, both, or none \citep{international1997international}. For this analysis, we only compare aspirin ($\Dec= 1$) and no aspirin ($\Dec=0$) and heparin is viewed as a covariate. In total, our analysis included 23 covariates $\X$, such as age, sex, level of consciousness, and neurological symptoms. The public data set and descriptions of the covariates are available in \citet{sandercock2011international}. The outcome $\Loss$ of interest is death at six-month follow-up. We exclude 984 patients from a preliminary study and patients with missing or unknown outcome data (153 patients), resulting in a final sample of 9154 patients in the non-treatment group and 9144 patients in the aspirin group.

For policy learning, we use the same setup as in \cref{sec:star}. Assuming the weights are correctly assigned, the resulting policies show a valid coverage for all values of $\trt$ (\cref{fig:stroke-criteria}). In this scenario the difference between treat and non-treatment is small (\cref{fig:star-obj}) and a majority of the resulting policies switch from treat non to treat all when $\trt$ is between 20\% and 25\%. 

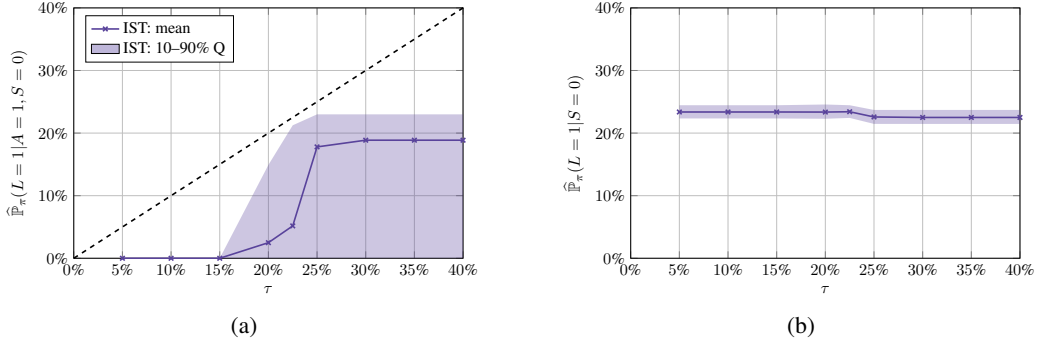
\begin{figure*}
\begin{subfigure}{0.47\textwidth}
    \centering
    \resizebox{\linewidth}{!}{ 
\begin{tikzpicture}
    \begin{axis}[
        xlabel={$\trt$},
        ylabel={$\widehat{\mathbb{P}}_{\pi}(L=1|A=1, S=0)$},
        yticklabel=\pgfmathparse{\tick*100}\pgfmathprintnumber{\pgfmathresult}\%,
        xticklabel=\pgfmathparse{\tick*100}\pgfmathprintnumber{\pgfmathresult}\%,
        xmin=0, xmax=0.4,
        ymin=0, ymax=0.4,
        legend pos=north west
    ]

    \addplot[color=Violet,  mark = x, line width=1pt] table[x=beta, y=mean constr, col sep=comma] {tikz/data/stroke.csv}; 
    \addlegendentry{IST: mean}
    \addlegendimage{area legend, fill=Violet!30}
    \addlegendentry{IST: 10–90\% Q}

    \addplot[name path=upper, draw=none] table[x=beta, y=constr_q90, col sep=comma] {tikz/data/stroke.csv};
    \addplot[name path=lower, draw=none] table[x=beta, y=constr_q10, col sep=comma] {tikz/data/stroke.csv};    
    \addplot[Violet, opacity=0.3] fill between[of=upper and lower];  

    \addplot [
        domain=0.0:0.5, 
        samples=10, 
        color=black,
        line width=1pt,
        dashed
        ]
        {x};

    \end{axis}
\end{tikzpicture}
    }
    \caption{}
    \label{fig:stroke-criteria}
\end{subfigure}
\hfill
\begin{subfigure}{0.47\textwidth}
    \centering
    \resizebox{\linewidth}{!}{ 
\begin{tikzpicture}
    \begin{axis}[
        xlabel={$\trt$},
        ylabel={$\widehat{\mathbb{P}}_{\pi}(L = 1|S=0)$},
        yticklabel=\pgfmathparse{\tick*100}\pgfmathprintnumber{\pgfmathresult}\%,
        xticklabel=\pgfmathparse{\tick*100}\pgfmathprintnumber{\pgfmathresult}\%,
        xmin=0, xmax=0.4,
        ymin=0, ymax=0.4,
        legend pos=north east
    ]

    \addplot[color=Violet,  mark = x, line width=1pt] table[x=beta, y=mean obj, col sep=comma] {tikz/data/stroke.csv}; 

    \addplot[name path=upper, draw=none] table[x=beta, y=obj_q90, col sep=comma] {tikz/data/stroke.csv};
    \addplot[name path=lower, draw=none] table[x=beta, y=obj_q10, col sep=comma] {tikz/data/stroke.csv};    
    \addplot[Violet, opacity=0.3] fill between[of=upper and lower];  

    \end{axis}
\end{tikzpicture}
    }
    \caption{}
    \label{fig:stroke-obj}
\end{subfigure}
\caption{The estimated probability of treatment risk, $\widehat{\Prob}_{\pi}(L = 1|\Dec=1, S=0)$ (a), and probability of nonrecovery, $\widehat{\Prob}_{\pi}(L = 1|S=0)$ (b) under policy $\policy$ for different values of the treatment risk tolerance $\trt$ for the IST dataset.}
\label{fig:stroke}
\end{figure*}

\section{Alternative Method for Average Guarantee}
\label{sec:app_alt_method}

We will now describe how to find a policy $\policy(\X)$ that controls the expected treatment risk of the learned policy,
\begin{equation}
\E[ \Tr(\policy) ] \leq \trt,
\label{eq:treatmentriskaverageguarantee}
\end{equation}
with a probabilty no less than $1-\miscoveragerate$.  Here we use a proof technique that uses similar steps as in  \cite{angelopoulos2022conformal}.

Randomly split the data $\settrial$ into three parts $\setpolicy$, $\setB$, and $\setrisk$ with $\npolicy$, $\nB$ and $\nrisk$ samples, respectively. For clarity and simplicity, we index the samples within each dataset starting from 1 and continuing up to their respective total number of samples. The first dataset, $\setpolicy$, is used to generate a set of rule-based policies, i.e., $\policy(\X;\trtparam), t \in (0,1)$, following the formulation in (\ref{eq:minimization_policy}) of the main paper.

The set $\setB$ is used to construct an upper bound $\overline{V}^{\alpha}(\trtparam)$. Let $\overline{V}^{\alpha}(\trtparam)$ be the $(1-\alpha)(1 + 1/\nB)$-th empirical quantile of $\{V_i(\trtparam) \;|\; 1 \leq i \leq \nB \}$, where $ V(\trtparam) = L \cdot \frac{\ind{\Dec =1}}{\probpolicy(\Dec=1|\Sampling=0)} \cdot \overline{\Weight}^{\Gamma}(\trtparam)$.

Finally, using $\setrisk$ we define $R_n(\trtparam) = \widehat{\E}_n [V(\trtparam)]$.

\begin{theorem}
\label{sec:average_theorem}
For any specified treatment risk tolerance $\trt$, define the empirical tolerance:
\begin{equation}
\trtparam_{\nrisk} = \argmin_{t \in (0,1)} \; \Ehat_{\nrisk} \left[L \cdot \overline{\Weight}^{\Gamma}(t)\right] \quad \text{subject to} \; \trt \geq \frac{1}{\nrisk+1} \left( \nrisk R_n(\trtparam_n) + \overline{V}^{\alpha}(\trtparam_n)\right).
    \label{eq:choose_trt_average} 
\end{equation}
The policy $\policy(\X; \trtparam_n )$ is certified to control the treatment risk in expectation according to (\ref{eq:treatmentriskaverageguarantee}) with a probability no less than $1-\miscoveragerate$. 
\end{theorem}

\begin{proof}
For notational convenience, we omit the conditioning on $\Sampling = s$ in the expressions that follow. From \cref{lem:partialidentifiability} we have $\Tr(\policy) \leq \E[V]$. Let $M$ be the (multi)set of all possible permutations of elements from $\{V_{1}, \dots, , V_{n}, V_{n + 1}\}$, where $V_{n + 1}$ represents a future sample. We have
\begin{equation*}
    V_{n+1}(\trtparam) | M \sim \text{Uni}(\{ V_{1}, \dots, V_{n}, V_{n+1} \}),
\end{equation*}
for any fixed $\trtparam$. The expectation is therefore
\begin{align*}
    \E[V_{n+1}(\trtparam) | M ] &= \frac{1}{n+1} \sum^{n+1}_{i=1} V_{i}(\trtparam) \\
    &= \frac{1}{n+1} \left( n R_n(\trtparam) + V_{n+1}(\trtparam) \right).
\end{align*}
Define event indicator that the upper bound $\overline{V}^{\alpha}(\trtparam)$ is valid, $$E =\ind{V_{n+1}(\trtparam) \leq \overline{V}^{\alpha}(\trtparam)}.$$
Then 
\begin{equation*}
\Prob [E=1] \geq 1 - \alpha,
\label{eq_appx:weight_bound_guarantee}
\end{equation*}
see \citet{vovk2005algorithmic}; \citet[thm.~2.1]{lei2018distribution}.
Thus, if the event holds, it follows that 
\begin{align*}
   \E[V_{n+1}(\trtparam) | M, E=1 ] \leq \frac{1}{n+1} \left( n R_n(\trtparam) + \overline{V}^{\alpha}(\trtparam) \right).
\end{align*}
Choose any $\trtparam_n$ such that the right-hand side is less than $\trt$, i.e.,
\begin{align*}
    \E[V_{n+1}&(\trtparam_n)| M, E=1 ]
    \leq \frac{1}{n+1} \left( n R_n(\trtparam_n) + \overline{V}^{\alpha}(\trtparam_n) \right) 
    \leq \trt.
\end{align*}

Using law of total expectation, we have
\begin{align*}
   \E[V_{n+1}(\trtparam_n) | E=1 ] \leq \trt ,
\end{align*}
which is holds with a probability of at least $1-\alpha$.

\end{proof}

\subsection{Synthetic Data}
We generate $|\setdata| = \npolicy+\nB+\nrisk=1000+200+800$ samples from \cref{eq:syn_gen_x,eq:synthetic_propensity,eq:syn_gen_y}. Similarly to the main method, the set $\setpolicy$ is used to learn the policies $\policy(\X; \trtparam)$ in \cref{eq:minimization_policy} but for both $\Gamma = {1, 2}$. The set $\setB$ is used to estimate an upper bound $\bar{V}^{\alpha}(\trtparam)$, which holds with high probability ($\alpha = 0.01$). The last set $\setrisk$ is used to select $\trtparam_n$ in \cref{eq:choose_trt_average}. $\trt$ is varied over the range $[0.05, 0.4]$, spaced with 0.05 at each step. 

The expected treatment risk, $\E[\Tr(\policy)]$, and the expected population risk, $\E[\Popr(\policy)]$ are estimated by Monte Carlo simulations using 1000 runs, and the results are shown in \cref{fig:app-synthetic-gamma}. In \cref{fig:app-synthetic-gamma-criteria}, it is evident that $\trt$ effectively regulates the treatment risk on average. In \cref{fig:app-synthetic-gamma-obj} the trade-off between lowering population risk and controlling treatment risk is visible. 

As in the main paper, a confounding variable $\Unobserved$ is included in this case. The decision policy is learned using the same approach as previously described. In \cref{fig:app-synthetic-vio-criteria}, it is evident that for $\Gamma = 1$, the guarantee that the treatment risk should be below $\trt$ on average is not satisfied. In contrast, $\Gamma = 2$ provides a valid model for the propensity scores and the guarantee is valid. In \cref{fig:app-synthetic-vio-obj} the trade-off between minimizing the probability of nonrecovery and minimizing the treatment risk remains visible.

\begin{figure*}
\begin{subfigure}{0.47\textwidth}
    \centering
    \resizebox{\linewidth}{!}{ 
        \begin{tikzpicture}
    \begin{axis}[
        xlabel={$\trt$},
        ylabel={$\E[\Tr(\policy)]$},
        yticklabel=\pgfmathparse{\tick*100}\pgfmathprintnumber{\pgfmathresult}\%,
        xticklabel=\pgfmathparse{\tick*100}\pgfmathprintnumber{\pgfmathresult}\%,
        xmin=0, xmax=0.4,
        ymin=0, ymax=0.4,
        legend pos=north west
    ]

    \addplot[color=RoyalBlue, mark = o, line width=1pt] table[x=beta, y=mean constr, col sep=comma] {tikz/data/syn_mean_true1.csv}; 
    \addlegendentry{$\Gamma = 1$}
    
    \addplot[color=Peach, mark = star, line width=1pt] table[x=beta, y=mean constr, col sep=comma] {tikz/data/syn_mean_true2.csv};
    \addlegendentry{$\Gamma = 2$}

    \addplot [
        domain=0.0:0.4, 
        samples=10, 
        color=black,
        line width=1pt,
        dashed
        ]
        {x};
    \end{axis}
\end{tikzpicture}
    }
    \caption{}
    \label{fig:app-synthetic-gamma-criteria}
\end{subfigure}  
\hfill
\begin{subfigure}{0.47\textwidth}
    \centering
    \resizebox{\linewidth}{!}{ 
        \begin{tikzpicture}
    \begin{axis}[
        xlabel={$\trt$},
        ylabel={$\E[\Popr(\policy)]$},
        yticklabel=\pgfmathparse{\tick*100}\pgfmathprintnumber{\pgfmathresult}\%,
        xticklabel=\pgfmathparse{\tick*100}\pgfmathprintnumber{\pgfmathresult}\%,
        xmin=0, xmax=0.4,
        ymin=0, ymax=1.0,
        legend pos=north east
    ]

    \addplot[color=RoyalBlue, mark = o, line width=1pt] table[x=beta, y=mean obj, col sep=comma] {tikz/data/syn_mean_true1.csv}; 
    \addlegendentry{$\Gamma = 1$}

    \addplot[color=Peach, mark = star, line width=1pt] table[x=beta, y=mean obj, col sep=comma] {tikz/data/syn_mean_true2.csv}; 
    \addlegendentry{$\Gamma = 2$}


    \end{axis}
\end{tikzpicture}
    }
    \caption{}
    \label{fig:app-synthetic-gamma-obj}
\end{subfigure}
\caption{The expected treatment risk, $\E[\Tr(\policy)]$, and the expected population risk, $\E[\Popr(\policy)]$,  for different values of the treatment risk tolerance $\trt$ and different confounding assumptions ($\Gamma = {1 ,2}$) using known weights. (a) Expected treatment risk falls below any given $\trt$ for both $\Gamma = {1 ,2}$. (b) The trade-off between the expected population risk and different choices of $\trt$. $\Gamma = 2$ is more conservative.}
\label{fig:app-synthetic-gamma}
\end{figure*}
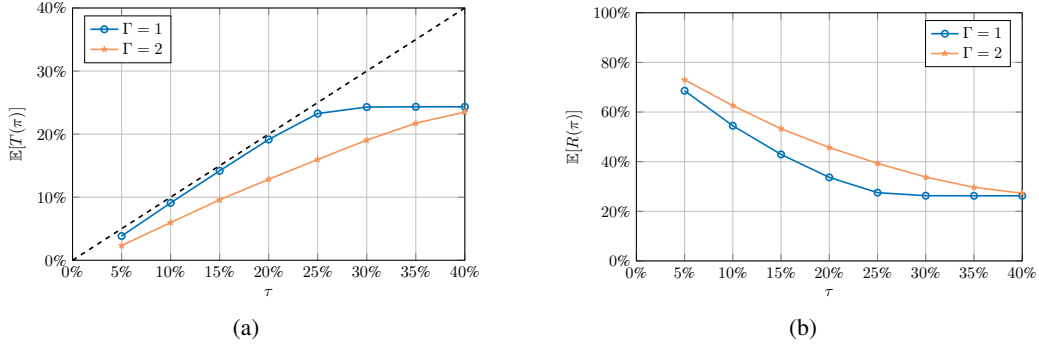

\begin{figure*}
\begin{subfigure}{0.47\textwidth}
    \centering
    \resizebox{\linewidth}{!}{ 
        \begin{tikzpicture}
    \begin{axis}[
        xlabel={$\trt$},
        ylabel={$\E[\Tr(\policy)]$},
        yticklabel=\pgfmathparse{\tick*100}\pgfmathprintnumber{\pgfmathresult}\%,
        xticklabel=\pgfmathparse{\tick*100}\pgfmathprintnumber{\pgfmathresult}\%,
        xmin=0, xmax=0.4,
        ymin=0, ymax=0.4,
        legend pos=north west
    ]

    \addplot[color=RoyalBlue, mark = o, line width=1pt] table[x=beta, y=mean constr, col sep=comma] {tikz/data/syn_mean_gamma1.csv};
    \addlegendentry{$\Gamma = 1$}
    
    \addplot[color=Peach, mark = star, line width=1pt] table[x=beta, y=mean constr, col sep=comma] {tikz/data/syn_mean_gamma2.csv};
    \addlegendentry{$\Gamma = 2$}

    \addplot [
        domain=0.0:0.4, 
        samples=10, 
        color=black,
        line width=1pt,
        dashed
        ]
        {x};
    \end{axis}
\end{tikzpicture}
    }
    \caption{}
    \label{fig:app-synthetic-vio-criteria}
\end{subfigure}  
\hfill
\begin{subfigure}{0.47\textwidth}
    \centering
    \resizebox{\linewidth}{!}{ 
        \begin{tikzpicture}
    \begin{axis}[
        xlabel={$\trt$},
        ylabel={$\E[\Popr(\policy)]$},
        yticklabel=\pgfmathparse{\tick*100}\pgfmathprintnumber{\pgfmathresult}\%,
        xticklabel=\pgfmathparse{\tick*100}\pgfmathprintnumber{\pgfmathresult}\%,
        xmin=0, xmax=0.4,
        ymin=0, ymax=1.0,
        legend pos=north east
    ]

    \addplot[color=RoyalBlue, mark = o, line width=1pt] table[x=beta, y=mean obj, col sep=comma] {tikz/data/syn_mean_gamma1.csv}; 
    \addlegendentry{$\Gamma = 1$}

    \addplot[color=Peach, mark = star, line width=1pt] table[x=beta, y=mean obj, col sep=comma] {tikz/data/syn_mean_gamma2.csv}; 
    \addlegendentry{$\Gamma = 2$}

    \end{axis}
\end{tikzpicture}
    }
    \caption{}
    \label{fig:app-synthetic-vio-obj}
\end{subfigure}
\caption{The expected treatment risk, $\E[\Tr(\policy)]$, and the expected population risk, $\E[\Popr(\policy)]$,  for different values of the treatment risk tolerance $\trt$ and different confounding assumptions ($\Gamma = {1 ,2}$). The weights are confounded. (a) For $\Gamma = 1$, the policy is invalid across a range of $\trt$ values, while $\Gamma = 2$ provides a valid model for the propensity scores, ensuring that the resulting policies remain valid across all $\trt$ values. (b) The trade-off between the expected population risk and different choices of $\trt$.}
\label{fig:app-synthetic-vio}
\end{figure*}
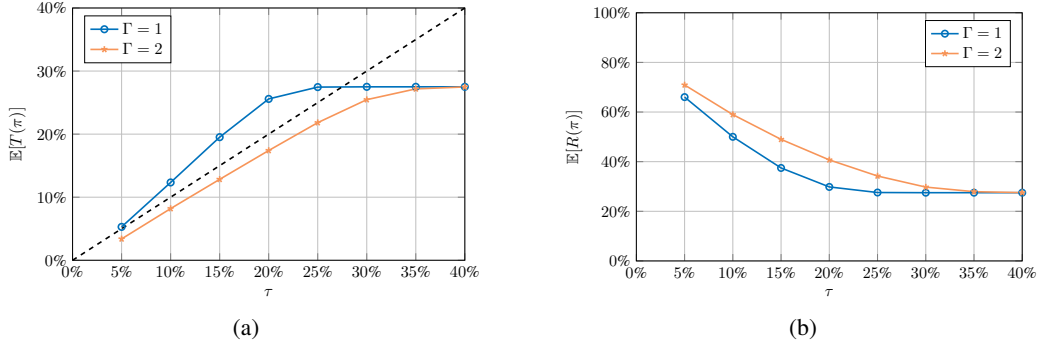

\subsection{STAR Data}
The preprocessing of the data is the same as in \cref{sec:star}. Of the final dataset comprising 4,218 students, 60 percent is allocated for policy construction in \cref{eq:minimization_policy}. The remaining samples are split equally: 20 percent are used to estimate an upper bound $\overline{V}^{\alpha}(\trtparam)$ that is maintained with high probability ($\alpha = 0.01$) and for evaluation, while the final 20 percent are dedicated to selecting $\trtparam_n$ in \cref{eq:choose_trt_average}. The parameter $\trt$ ranges from $0.05$ to $0.45$, increasing in steps of $0.05$. 

The resulting policies remain valid for all values of $\trt$ (\cref{fig:app-star-criteria}), assuming the weights are assigned correctly. The trade-off between minimizing the total expected loss and the expected loss for the treated remains evident (\cref{fig:app-star-obj}).

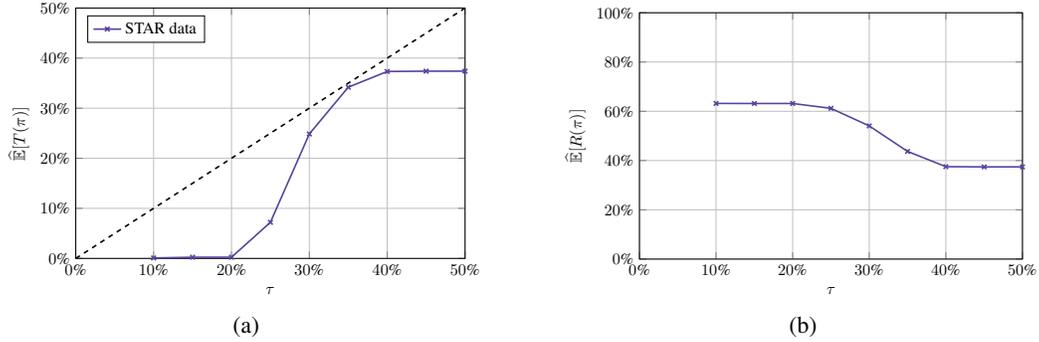
\begin{figure*}
\begin{subfigure}{0.47\textwidth}
    \centering
    \resizebox{\linewidth}{!}{ 
\begin{tikzpicture}
    \begin{axis}[
        xlabel={$\trt$},
        ylabel={$\widehat{\E}[\Tr(\policy)]$},
        yticklabel=\pgfmathparse{\tick*100}\pgfmathprintnumber{\pgfmathresult}\%,
        xticklabel=\pgfmathparse{\tick*100}\pgfmathprintnumber{\pgfmathresult}\%,
        xmin=0, xmax=0.5,
        ymin=0, ymax=0.5,
        legend pos=north west
    ]

    \addplot[color=Violet,  mark = x, line width=1pt] table[x=beta, y=mean constr, col sep=comma] {tikz/data/star_mean.csv}; 
    \addlegendentry{STAR data}
    
    \addplot [
        domain=0.0:0.5, 
        samples=10, 
        color=black,
        line width=1pt,
        dashed
        ]
        {x};

    \end{axis}
\end{tikzpicture}
    }
    \caption{}
    \label{fig:app-star-criteria}
\end{subfigure}  
\hfill
\begin{subfigure}{0.47\textwidth}
    \centering
    \resizebox{\linewidth}{!}{ 
\begin{tikzpicture}
    \begin{axis}[
        xlabel={$\trt$},
        ylabel={$\widehat{\E}[\Popr(\policy)]$}, 
        yticklabel=\pgfmathparse{\tick*100}\pgfmathprintnumber{\pgfmathresult}\%,
        xticklabel=\pgfmathparse{\tick*100}\pgfmathprintnumber{\pgfmathresult}\%,
        xmin=0, xmax=0.5,
        ymin=0, ymax=1,
        legend pos=north east
    ]

    \addplot[color=Violet,  mark = x, line width=1pt] table[x=beta, y=mean obj, col sep=comma] {tikz/data/star_mean.csv}; 

    \end{axis}
\end{tikzpicture}
    }
    \caption{}
    \label{fig:app-star-obj}
\end{subfigure}
\caption{The estimated probability of the loss for treated individuals, $\widehat{\E}[\Tr(\policy)]$ (a) and the estimated probability of the loss, $\widehat{\E}[\Popr(\policy)]$ (b), for different values of the design parameter $\trt$ for the STAR dataset.}
\label{fig:app-star}
\end{figure*}

\end{document}